\documentclass{article}

\usepackage[main,final,nonatbib]{neurips_2026}

\usepackage[utf8]{inputenc} %
\usepackage[T1]{fontenc}    %
\usepackage{hyperref}       %
\usepackage{url}            %
\usepackage{booktabs}       %
\usepackage{amsfonts}       %
\usepackage{nicefrac}       %
\usepackage{microtype}      %
\usepackage{xcolor}         %
\pdfoutput=1
\usepackage{amsmath,amsthm,wrapfig}
	\usepackage[utf8]{inputenc} 
\usepackage{amsmath, amssymb,bm, cases, mathtools, thmtools}
\usepackage{verbatim}
\usepackage{graphicx}\graphicspath{{figures/}}
\usepackage{multicol}
\usepackage{tabularx}
\usepackage{mathrsfs} 
\usepackage{caption}
\usepackage{algorithm}
\usepackage{algorithmicx}
\usepackage[noend]{algpseudocode}
\usepackage{array,booktabs,arydshln,xcolor}

\usepackage[T1]{fontenc}
\usepackage{inconsolata}
\usepackage{dsfont}
\usepackage{hyperref}
\usepackage{listings} %

\usepackage{enumitem}
\usepackage{booktabs}       %
\usepackage{nicefrac}       %

\usepackage{lineno}
\usepackage{bbm}

\usepackage{caption}
\usepackage{subcaption}

\usepackage{datetime}

\DeclareMathAlphabet\EuRoman{U}{eur}{m}{n}
\SetMathAlphabet\EuRoman{bold}{U}{eur}{b}{n}

\usepackage{amsthm}

\declaretheorem[style=plain,numberwithin=section,name=Theorem]{theorem}

\declaretheorem[style=plain,sibling=theorem,name=Lemma]{lemma}
\declaretheorem[style=plain,sibling=theorem,name=Assumption]{assumption}

\declaretheorem[style=plain,sibling=theorem,name=Corollary]{corollary}

\declaretheorem[style=definition,sibling=theorem,name=Definition]{definition}

\declaretheorem[style=remark,qed=$\triangleleft$,sibling=theorem,name=Remark]{remark}
\numberwithin{theorem}{section}

\usepackage{xparse}
\usepackage{xstring}
\usepackage{xspace}
\usepackage{enumitem}
\usepackage{thmtools} 

\newcommand{\Dist}{\mathcal D}
\newcommand{\dataspace}{\mathcal X}
\newcommand\optparen[1]{\ifthenelse{\equal{#1}{}}{}{(#1)}}

\newcommand{\Naturals}{\mathbb{N}}

\newcommand{\Reals}{\mathbb{R}}

\DeclareMathOperator*{\newlim}{\mathrm{lim}\vphantom{\mathrm{infsup}}}
\DeclareMathOperator*{\newmin}{\mathrm{min}\vphantom{\mathrm{infsup}}}
\DeclareMathOperator*{\newmax}{\mathrm{max}\vphantom{\mathrm{infsup}}}
\DeclareMathOperator*{\newinf}{\mathrm{inf}\vphantom{\mathrm{infsup}}}
\DeclareMathOperator*{\newsup}{\mathrm{sup}\vphantom{\mathrm{infsup}}}
\renewcommand{\lim}{\newlim}
\renewcommand{\min}{\newmin}
\renewcommand{\max}{\newmax}
\renewcommand{\inf}{\newinf}
\renewcommand{\sup}{\newsup}

\newcommand{\ProbMeasures}[1]{\Delta(#1)}

\def\EE{\mathbb{E}}

\newcommand{\e}{\mathrm{e}}

\newcommand{\Alg}{\mathcal{A}}

\newcommand{\lcrx}[4][{-1}]{
	\IfEq{#1}{-1}{\left #2 {{{{#3}}}} \right #4}{
   	\IfEq{#1}{0}{#2 {{{{#3}}}} #4}{
	\IfEq{#1}{1}{\bigl #2 {{{{#3}}}} \bigr #4}{
	\IfEq{#1}{2}{\Bigl #2 {{{{#3}}}} \Bigr #4}{
	\IfEq{#1}{3}{\biggl #2 {{{{#3}}}} \biggr #4}{
	\IfEq{#1}{4}{\Biggl #2 {{{{#3}}}} \Biggr #4}{
    \GenericWarning{"4th argument to lcrx must be -1, 0, 1, 2, 3, or 4"}
    }}}}}}}

\newcommand{\indic}[1]{\mathds{1}\left[#1\right]}

\newcommand{\starnum}{\mathfrak{s}}

\newcommand{\Adap}{\textsc{Adap}\xspace}
\newcommand{\SAP}{\textsc{SampleAware}\xspace}

\newcommand{\DAP}{\textsc{DAP}\xspace}
\newcommand{\Uni}{\textsc{Uniform}\xspace}

\newcommand{\vercost}{c_{\text{ver}}}

\newcommand{\rewcost}{c_{\text{rew}}}

\newcommand{\cmin}{c_{\min}}

\newcommand{\conceptclass}{\mathcal{H}}
\newcommand{\Pool}{\mathcal{P}}
\usepackage{xspace}
\usepackage[font=footnotesize,labelfont=bf,skip=4pt,belowskip=-6pt]{caption}
\usepackage[font=scriptsize]{subcaption}
\renewcommand{\arraystretch}{0.95}   %
\usepackage{booktabs}  
\usepackage[%
        minnames=1,maxnames=99,maxbibnames=99,minalphanames=1,maxalphanames=3,
		style=alphabetic,
		sorting=nyt,
		sortcites=false, %
		doi=false,url=false,
		uniquename=init,
		giveninits=true,
		hyperref,natbib,
		backend=biber]{biblatex}

\renewbibmacro{in:}{%
	\ifentrytype{article}{}{\printtext{\bibstring{in}\intitlepunct}}}
\usepackage[capitalize]{cleveref}
\usepackage{hyperref}
\usepackage{titletoc}
\hypersetup{
	linktocpage=true,
	colorlinks=true,				
	linkcolor=[rgb]{.7,0,0},				
	citecolor=magenta,				
	urlcolor=[rgb]{.7,0,.7},
}

\title{Adaptive Generate-Rank-Verify: \\ Inference-Time Search with Costly Verification}

\usepackage{titlesec}
\titlespacing{\section}{0pt}{\parskip}{0pt}
\titlespacing{\subsection}{0pt}{\parskip}{0pt}
\titlespacing{\subsubsection}{0pt}{\parskip}{0pt}

\renewcommand{\epsilon}{\varepsilon}

\usepackage[normalem]{ulem}

\usepackage[colorinlistoftodos]{todonotes}
\def\[#1\]{\begin{equation*}\begin{aligned}#1\end{aligned}\end{equation*}}
\def\*[#1\*]{\begin{align*}#1\end{align*}}

\author{%
  Shaddin Dughmi$^{\star 1}$ \quad Mahdi Haghifam$^{\star 2}$ \quad Yusuf Hakan Kalayci$^{\star 3}$ \\
  $^{1}$University of Southern California \\
  $^{2}$Toyota Technological Institute at Chicago \\
  $^{3}$IDEAL Institute, University of Illinois Chicago\\
  $^\star$Authors listed in alphabetical order
}

\begin{document}

\maketitle

\begin{abstract}
 Many inference-time language-model pipelines combine a cheap reward signal with an expensive verifier, such as exact answer checking in mathematical reasoning or hidden-test execution in code generation. 
 We formalize this setting using a learning-theoretic lens as
\emph{generative active search}: a cost-sensitive first-positive search problem
in which a policy adaptively samples candidates from an unknown distribution,
observes cheap scores, and pays for verifier labels until it finds a positive
example.
For a fixed prompt, the generator and reward model induce two unknown objects: a distribution over reward scores and a score-conditioned success function. When these quantities are known, we characterize the distribution-aware optimal policy using a dynamic programming approach. In the realistic and practical setting where both the score distribution and success function are unknown, we propose ADAP, a shellwise adaptive generate-rank-verify algorithm that progressively increases the number of sampled responses and top-ranked verifications. Under the monotonicity assumption that higher reward scores are no less likely to pass verification, we show that ADAP achieves expected cost within a constant factor of the distribution-aware optimum. 
We complement this result with
learning-theoretic lower bounds, based on a centered star number, showing that
structural assumptions on the score--label relationship are necessary.
Experiments on mathematical reasoning and competitive programming validate the
predicted advantage over both fixed non-adaptive policies and difficulty-adaptive baselines.
\end{abstract}

\section{Introduction}

State-of-the-art large language models (LLMs) increasingly rely on 
\emph{inference-time search}: they generate multiple candidate solutions, use a  cheap but noisy reward signal to triage them, and reserve an expensive verifier for the few candidates that may be correct. This paradigm underlies recent 
progress on difficult reasoning tasks, including competitive programming and 
mathematical Olympiad exams \citep{deepmind2025geminiicpc, openai2024o1}, where verification may take the form of hidden-test execution, exact answer checking, or proof verification. Because each additional generation, reward score, and verifier call incurs \emph{cost}---in wall-clock latency, GPU compute, or API spend---these systems face a basic compute-allocation problem at inference time. In a nutshell: given an LLM, a cheap but noisy reward model, and a costly verifier, how should one use reward scores to decide when to keep sampling and when to spend verification calls, in order to find a verified-correct response with minimal total cost? This question has motivated a growing line of work on test-time compute allocation and adaptive inference-time search \citep{snell2024scaling, damani2025hardthink, zhai2026adaptive, 
raman2025adabon, kalayci2025optimalstopping, wan2025beacon, qu2026adaptive}.

The challenge is not merely that inference-time compute is costly; it is that the right allocation is inherently prompt-dependent. A natural non-adaptive baseline fixes a pair $(N_\text{rew},N_\text{ver})$: generate $N_\text{rew}$ candidate responses, rank them by reward score, and verify the top $N_\text{ver}$. However, no single choice of $(N_\text{rew},N_\text{ver})$ is appropriate across all prompts. As shown in \cref{fig:mk-iso}, the per-prompt optimal choice varies widely. Easy instances may require only a small candidate pool and one verification call, while hard instances may require orders of magnitude more generation and several verification attempts. Similar prompt-dependent variation in the value of 
test-time compute has also been observed in recent studies \citep{snell2024scaling,damani2025hardthink,zhai2026adaptive,raman2025adabon,
huang2026optimalbayesian}. Thus, any fixed policy must either overspend on easy  prompts or under-search on hard ones.

One natural response is to learn a predictor of per-prompt difficulty from historical data and use it to choose the inference-time budget  \citep{snell2024scaling,damani2025hardthink}. Such approaches can be effective when the deployment prompt distribution is stable and the learned predictor remains well calibrated. However, this assumption can fail under distribution shift: the mix of prompts may change over time, and the relationship between reward scores and verifier acceptance may vary across prompts or domains. In such settings, historical tuning or a single global calibration rule may not transfer reliably. This motivates online, prompt-dependent policies that adapt using only the reward scores and verification outcomes observed on the current prompt. 

These observations lead to the central question we aim to answer in this paper: Can an online policy use an uncalibrated reward model, adapt to prompt difficulty, and achieve near-optimal cost relative to a distribution-aware policy that knows the score distribution and reward–verifier relationship? We answer this question by formulating generate--rank--verify inference as a 
cost-sensitive search problem and by developing an adaptive 
algorithm with provable and empirical guarantees across various cost regimes.

\begin{figure}[ht]
    \centering
    \begin{minipage}[t]{0.495\linewidth}
        \centering
        \includegraphics[width=0.85\linewidth]{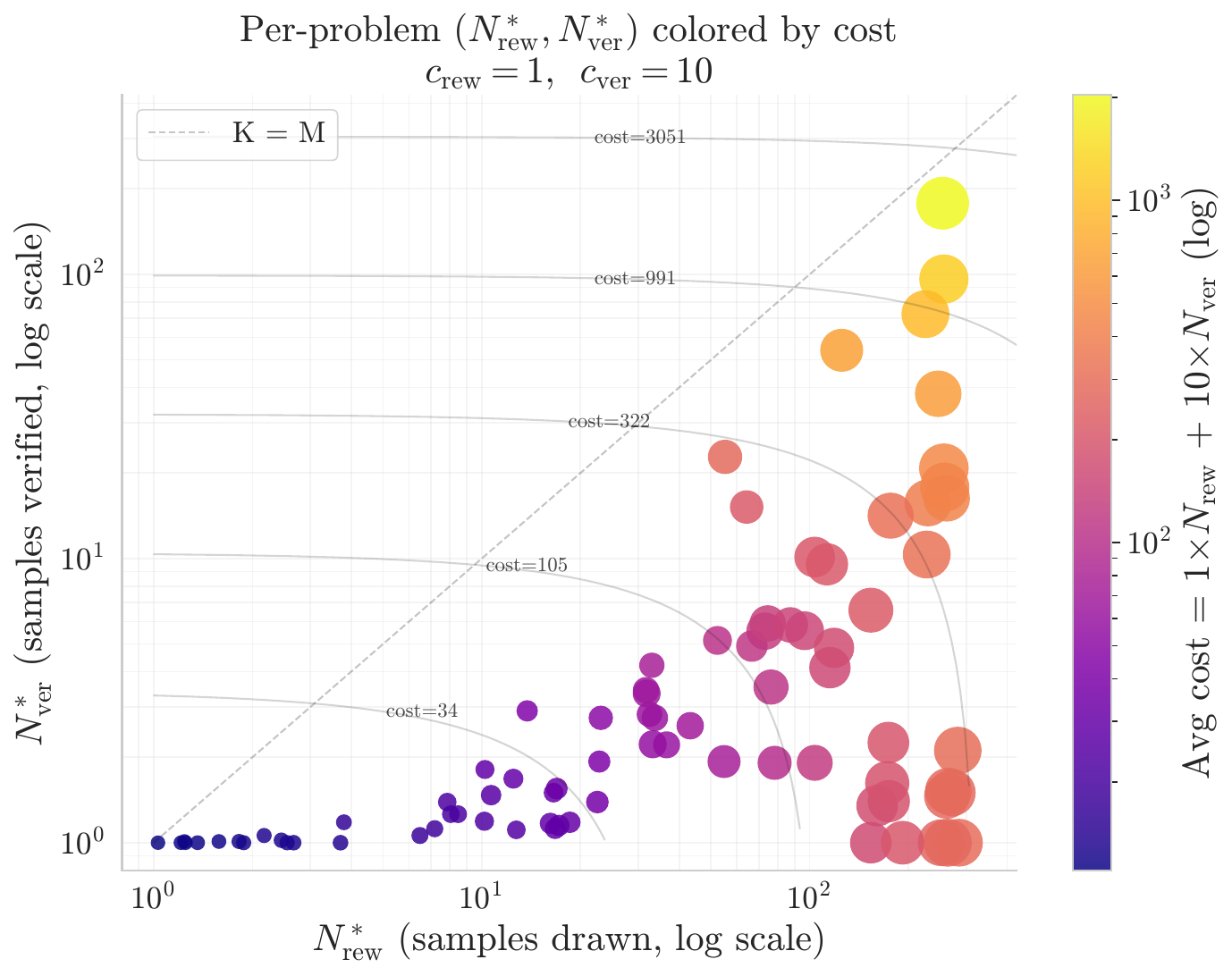}
        \captionof{figure}{Per-problem optimal non-adaptive choices \((\overline{N_{\text{rew}}^\star},\overline{N_{\text{ver}}^\star})\) on LiveCodeBench (\(100\) permutations each). For each problem, \(N_\text{rew}^\star\) is the number of generated candidates and \(N_\text{ver}^\star\) is the number of top-ranked candidates verified. Color encodes total cost; faint diagonals are iso-cost contours. The wide spread shows that fixed \((N_\text{rew},N_\text{ver})\) policies can over-spend on easy problems and under-spend on hard ones.}
        \label{fig:mk-iso}
    \end{minipage}
    \hfill
    \begin{minipage}[t]{0.495\linewidth}
        \centering
        \includegraphics[width=0.85\linewidth]{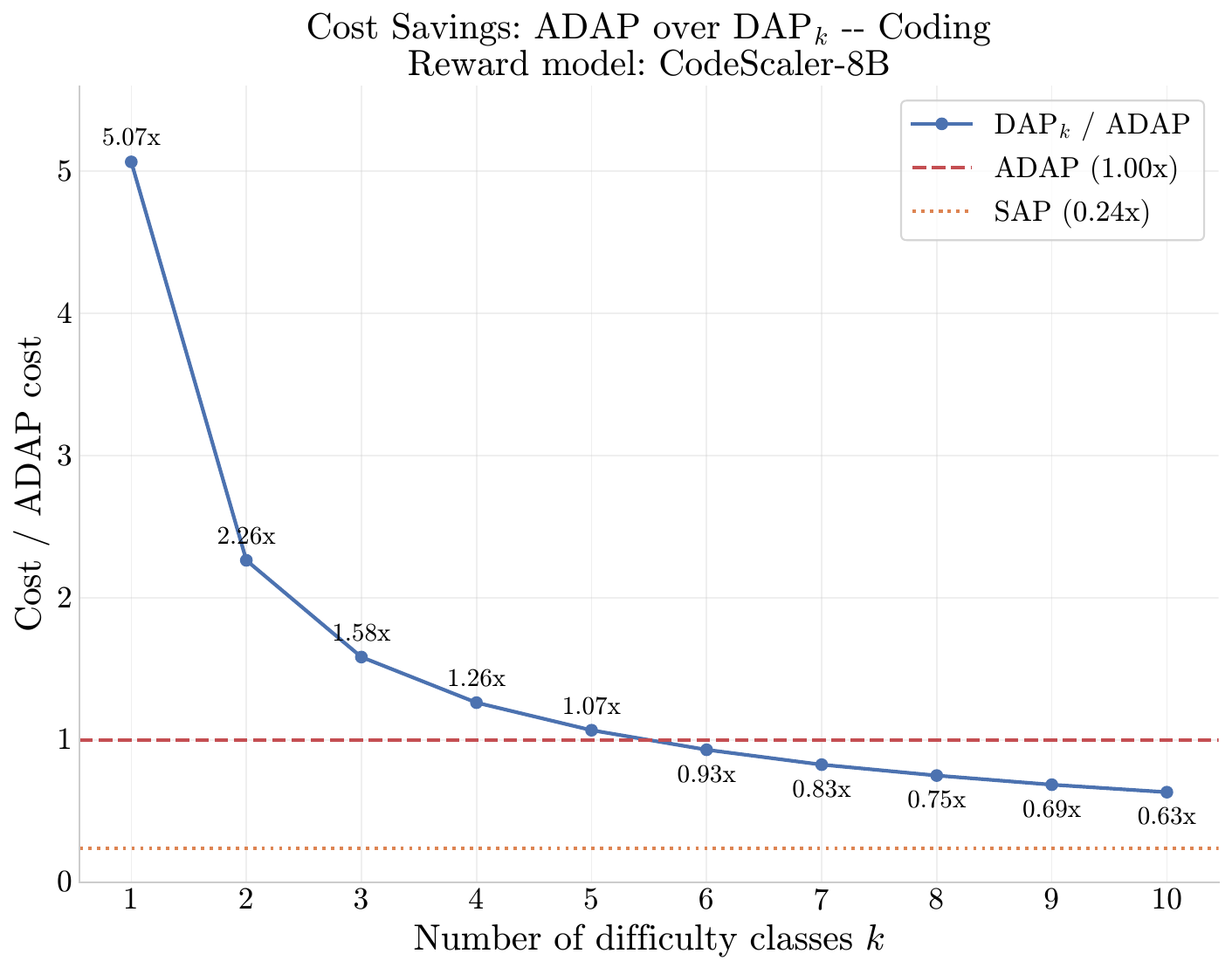}
        \captionof{figure}{
        The plot shows average cost ratios of two baselines---Difficulty Adaptive Policy ($\DAP_k$) and \SAP Policy---relative to \Adap. \SAP retroactively picks the minimum-cost generation count $N_{\text{rew}}$ and verification count $N_{\text{ver}}$ such that verifying the $N_{\text{ver}}$ highest-reward samples among $N_{\text{rew}}$ finds a correct answer. $\DAP_k$ estimates problem difficulty from the fraction of successful samples, partitions them into $k$ classes optimally, and assigns each the cheapest uniform strategy ensuring $100\%$ success.
        }
        \label{fig:performance-alg}
    \end{minipage}
\end{figure}

\subsection{Contributions}

  We study inference-time generate--rank--verify pipelines, where candidate 
responses are sampled from a language model, scored by a cheap reward model, and 
checked by an expensive verifier. Motivated by the compute-allocation question 
above, we formalize this setting as a cost-sensitive sequential decision problem, 
which we call \emph{active search}. Let \(\rewcost>0\) denote the cost of generating one candidate 
response and scoring it with the reward model, and let \(\vercost>0\) denote the 
cost of one verifier call. For each prompt \(x\), the generator and reward model induce an unknown
distribution \(\Dist_x \in \ProbMeasures{\Reals}\) over reward scores. The verifier induces an unknown
score-conditioned success function \(h_x^\star:\Reals \to [0,1]\), where \(h_x^\star(r)\) is the 
probability that a candidate with reward score \(r\) passes verification. 
A sound policy adaptively generates candidates, observes their reward scores, and 
chooses which candidates to verify. It pays $J = \rewcost N_{\rm rew} + \vercost N_{\rm ver},$  where \(N_{\rm rew}\) is the number of generated-and-scored candidates and 
\(N_{\rm ver}\) is the number of verifier calls, and it may stop only after 
observing a positive verifier label. This formulation captures the central tension 
from the introduction: the policy must decide, online and prompt-by-prompt, when 
to spend cost on more reward-scored samples and when to spend cost on verification.

\begin{enumerate}[leftmargin=*]
 \item \textbf{Optimal Distribution-Aware Benchmark.}  
    We first study the idealized setting in which the prompt-specific score
    distribution and reward--verifier relationship are known. This gives the
    benchmark for online adaptation and clarifies the ideal compute-allocation rule.
    We show that the optimal policy has a simple threshold form: it verifies a
    candidate exactly when its probability of passing the verifier is high enough
    to justify the verification cost (\cref{thm:distribution-aware}). 
    
   \item \textbf{A prompt-dependent online policy: \Adap.}
    In practice, the policy does not know \(\Dist_x\) or \(h^\star_x\), and therefore 
    cannot implement the distribution-aware policy directly. We propose \Adap, a shellwise 
    online algorithm that searches over \emph{dyadic} generation and verification scales (\cref{alg:expected-shellwise-search}).  
    \Adap progressively enlarges the candidate pool and verifies top-ranked 
    candidates, using the reward model only as a ranking signal. Under the natural 
    monotonicity assumption that \(h^\star_x\) is non-decreasing in the reward score, 
    we prove that \Adap achieves expected cost within a constant factor of the 
    distribution-aware optimum, uniformly over all feasible \((\Dist_x,h^\star_x)\) 
    (\cref{thm:expected-shellwise}).

     \item \textbf{Learning-theoretic foundations.} The guarantee for \Adap relies on the assumption that higher reward scores are more likely to pass verification. We show that some such inductive bias is unavoidable. Without structural assumptions on the reward--verifier relationship, any sound online policy can be forced to spend much more than the distribution-aware benchmark. We formalize this obstruction through the \emph{centered star number}, a learning-theoretic complexity measure for active search. For binary concept classes, we prove matching lower and upper bounds: the worst-case adaptivity gap is characterized, up to constants, by \(\min\{\starnum_0,\vercost/\rewcost\}\) where $\starnum_0$ is the centered star number. We also show a separation between active search and active learning \citep{balcan2006agnostic,balcan2016active,hanneke2015minimax}.

\item \textbf{Empirical validation.} We evaluate \Adap on HMMT mathematical reasoning and LiveCodeBench competitive
programming tasks, using exact answer matching and hidden-test execution as verifiers.
On the feasible subset of prompts with at least one correct sampled candidate,
\Adap matches the 100\% success rate of the cheapest fixed \((N_{\rm rew}, N_{\rm ver})\) policy that
also achieves 100\% success, while using substantially lower mean cost: \(2.9\times\)
lower on HMMT and \(5.5\times\) lower on LiveCodeBench in our current experiments.
At the same mean cost as \Adap, the best fixed policy succeeds on only \(74.4\%\)
and \(88\%\) of trials, respectively. \cref{fig:performance-alg} previews this comparison: \Adap remains competitive with \(\DAP_k\), an oracle
difficulty-stratified proxy for learning-based budget allocation, without using
historical prompt-difficulty labels; see \cref{sec:experiments} for details.
\end{enumerate}

\section{Problem Setup}
\label{sec:problem-setup}

Let $\mathcal X$ denote the prompt space and $\mathcal V$
 the vocabulary. The response space  $ \mathcal Y = \bigcup_{t=0}^{\infty} \mathcal V^t$
is the set of all sequences of arbitrary length. For each prompt $x\in\mathcal X$, the language model induces a distribution $\pi(\cdot \mid x)\in \ProbMeasures{\mathcal Y}$ over candidate responses. We write $Y\sim \pi(\cdot\mid x)$ for one sampled response.

We assume access to two \emph{evaluators} for a prompt-response pair. The first is a
reward model $\mathsf{RM}:\mathcal X\times \mathcal Y \to \mathbb R,$ which returns a real-valued score. The second is a verifier $\mathsf{Ver}:\mathcal X\times \mathcal Y \to \{0,1\},$ which returns the final correctness label. We say a response $y$ is \emph{accepted for prompt $x$} iff $\mathsf{Ver}(x,y)=1$. Crucially, the reward model is only a proxy for the verifier, and the reward model may be noisy or miscalibrated, and correctness is defined by the verifier.

For a fixed prompt $x$, the generator and evaluators induce a joint distribution over reward scores and verifier labels. In particular, if $Y\sim\pi(\cdot\mid x)$, define random variables
$    R_x= \mathsf{RM}(x,Y)$ and $
    V_x = \mathsf{Ver}(x,Y).
$
Then $(R_x,V_x)$ is a random element of $\mathbb R\times\{0,1\}$. We assume throughout that the prompt is feasible, meaning $\Pr_{Y\sim\pi(\cdot\mid x)}\!\left(\mathsf{Ver}(x,Y)=1\right) > 0$ since otherwise no generate-and-verify procedure can succeed. The key quantity connecting the reward model to verification is the conditional success probability
$
    h^\star_x(r) = \Pr(V_x=1 \mid R_x=r).
$
Equivalently, $h^\star_x(r)$ is the probability that a candidate with reward score 
$r$ will pass the verifier. When the prompt is clear from context, we suppress
the subscript $x$ and write $(R,V)$ and $h^\star$.

\subsection{Active Search}

The goal of active search is to find a
verified response while minimizing the total cost of generation, reward
querying, and verification. The policy sees neither which candidates the
verifier will accept nor how reward scores translate into acceptance
probabilities; it must learn enough about both, on the fly, to decide when
to spend a verification call and when to keep sampling. We formalize the
problem in two steps: \cref{def:active-search-instance} specifies the
problem instance induced by a prompt, and \cref{def:active-search}
specifies what it means for a policy to solve it.

\begin{definition}\label{def:active-search-instance}
Fix a prompt $x \in \mathcal X$. The generator $\pi(\cdot\mid x)$, reward model
$\mathsf{RM}$, and verifier $\mathsf{Ver}$ induce an \emph{active search
instance} $(\Dist_x, h^\star_x)$, where
\begin{itemize}[leftmargin=*]
    \item $\Dist_x \in \ProbMeasures{\Reals}$ is the marginal distribution of
    the reward score $R_x = \mathsf{RM}(x,Y)$ when $Y\sim\pi(\cdot\mid x)$;
    \item $h^\star_x : \Reals \to [0,1]$ is the score-conditioned success function
    $h^\star_x(r) = \Pr(V_x = 1 \mid R_x = r)$, where $V_x = \mathsf{Ver}(x,Y)$ when $Y\sim\pi(\cdot\mid x)$. We also assume that $\Pr(V_x=1)>0$.
\end{itemize}
\end{definition}

In practice, $\Dist_x$ is unknown and only accessible via sampling. Furthermore, the true probability that a candidate with score $r$ passes verification, denoted $h^\star_x(r)=\Pr(V_x=1\mid R_x=r)$, is also unknown. Therefore, active search is governed by these two unknown objects: the score distribution $\Dist_x$ and the conditional success function $h^\star_x$. Similar to PAC learning, we impose no assumptions on $\Dist_x$. However, we assume $h_x \in \mathcal{H}$ for a \emph{probabilistic concept class} $\mathcal{H} \subseteq [0,1]^{\Reals}$ \citep{kearns1994efficient}. This is not a strict calibration requirement on the reward model, but rather a necessary inductive bias that allows active search to succeed without exhaustively sampling the space (see \cref{sec:seperation}).

\begin{definition}\label{def:active-search}
Fix a probabilistic concept class $\mathcal H \subseteq [0,1]^{\mathbb R}$. Fix costs $\rewcost>0$ (generation+reward cost per sample) and $\vercost>0$ (verification cost per sample). A \emph{sound active search policy} is a randomized sequential procedure $\Alg$ that knows $\mathcal H,\rewcost,\vercost$ and interacts with an unknown pair $(\Dist,h^\star)$, where $\Dist\in\ProbMeasures{\mathbb R}$ and $h^\star\in\mathcal H$. {
The policy maintains a pool of generated but unverified candidates, represented
by pairs \((Y_i,R_i)\), where \(Y_i\) is the response and \(R_i\) is its reward
score. At each decision time, based on its history, \(\Alg\) chooses one of two
elementary actions: 

\begin{enumerate}[leftmargin=*]
    \item \textbf{Generate.} Draw one fresh response \(Y\sim\pi(\cdot\mid x)\),
    observe its reward score \(R=\mathsf{RM}(x,Y)\), pay cost \(\rewcost\), and
    add \((Y,R)\) to the unverified pool.

    \item \textbf{Verify.} Choose one pair \((Y_i,R_i)\) from the unverified pool,
    query its verifier label, pay cost \(\vercost\), and remove the pair from the
    pool. Conditional on \(R_i\), the revealed label has distribution
    \(V_i\sim\mathrm{Bernoulli}(h^\star(R_i))\). If \(V_i=1\), the policy stops
    and outputs \(Y_i\).
\end{enumerate}

Let \(N_{\rm rew}\) be the total number of generated-and-scored candidates, and
let \(N_{\rm ver}\) be the total number of verifier calls made before stopping.
The random total cost is
\[
    J(\Alg;h^\star,\Dist)
    :=
    \rewcost N_{\rm rew}+\vercost N_{\rm ver}.
\]
If \(\Alg\) never observes a positive verifier label, \(J(\Alg;h^\star,\Dist)=+\infty\). Finally note that the policy’s decisions are measurable with respect to the reward scores of unverified candidates, previously chosen verification indices, and observed verifier labels; it does not inspect the contents of unverified responses except through their reward scores.} 
\end{definition}

\section{Distribution-Aware Benchmark}
\label{sec:distribution-aware}

Before tackling the realistic setting where $\Dist$ and $h^\star$ are unknown, we first establish the fundamental limits of active search by analyzing an idealized algorithm with perfect knowledge of the environment. This distribution-aware optimum provides a strict lower bound on the expected cost and reveals a structural property of the optimal policy that will guide the design of our adaptive algorithm. 

\begin{definition}\label{def:distribution-aware}
A \emph{distribution-aware active search policy} is any valid active search
policy in the sense of \Cref{def:active-search} that is additionally given
perfect knowledge of \((\Dist,h^\star)\). Let \(\Pi_{\mathrm{DA}}(\Dist,h^\star)\) denote the
class of all such policies. The optimal distribution-aware expected cost is
\[
    J^\star(h^\star,\Dist):= 
    \inf_{\pi\in\Pi_{\mathrm{DA}}(\Dist,h^\star)}
    \EE
    \left[
        \rewcost\,N_{\mathrm{rew}}(\pi)
        +
        \vercost\,N_{\mathrm{ver}}(\pi)
    \right],
\]
where policies that never output a verified positive have cost \(+\infty\).
\end{definition}

We then present the main result whose proof can be found in \cref{sec:proof-distribution-aware} as \cref{lem:optimal-streaming-oracle,lem:distribution-aware-online-optimal}. 
\begin{theorem}
\label{thm:distribution-aware}
Fix a prompt \(x\), and suppress the subscript \(x\). Let \(R\sim \Dist\), and let
\(h^\star(r)=\Pr(V=1\mid R=r)\). Assume \(\Pr(V=1)>0\). Define \(\tau^\star\in(0,1)\) as the unique solution of
\[
    \vercost\,\mathbb E_{R\sim\Dist}[\max\{h^\star(R)-\tau^\star,0\}]
    =
    \tau^\star\,\rewcost .
\]
Then  the  optimal expected cost  with  knowledge of $(h^\star,\Dist)$ is \(J^\star(h^\star,\Dist)=\vercost/\tau^\star\). Moreover, an  optimal fully-sequential policy  is as follows: after generating a candidate with reward score \(r\), it verifies immediately if \(h^\star(r)>\tau^\star\), discards it if \(h^\star(r)<\tau^\star\), and breaks ties arbitrarily if $h^\star(r)=\tau^\star$.
\end{theorem}

\cref{thm:distribution-aware} shows that the optimal distribution-aware policy is a simple threshold rule: verify a candidate if and only if its conditional success probability exceeds a single break-even threshold. The break-even threshold  is based on comparing the immediate chance of success and the continuation value of discarding it and drawing again.

\section{ADAP: Adaptive Policy}
\label{sec:non-dec-alg}

\cref{thm:distribution-aware} reveals the structural property of the optimal distribution-aware policy. However, in practice, the true functions governing the environment are unknown. To bridge this gap, we introduce \Adap, an 
adaptive policy that searches for the right threshold scale online. The policy does not require prior knowledge of $\mathcal{D}_x$ or $h^\star_x$; it only relies on the natural monotonicity assumption that higher reward scores are more likely to correspond to 
correct outputs. In \cref{app:reward-signal}, we empirically validate this assumption on both coding and math benchmarks, both in aggregate and at the per-prompt level.

\begin{assumption} \label{assumption:non-dec}
We assume $h^\star \in \conceptclass_{\text{non-dec}} $ where  $ \conceptclass_{\text{non-dec}} = \{h: \Reals \to [0,1] \mid \forall~r_1\leq r_2~ h(r_1)\leq h(r_2)\}$ is the family of non-decreasing functions.
\end{assumption}
This monotonicity serves as our primary inductive bias, ensuring that the reward model's rankings are informative. As we will formally prove in \cref{sec:seperation}, such structural assumptions are not technical conveniences, but fundamental requirements. Now, we are ready to formally state the performance guarantee of \Adap. After that we give an  overview of the proof (Formal proof given in \cref{sec:proof-adap})

\begin{algorithm}[h]
\caption{$\Alg_{\Adap}$: Shellwise Active Search}
\label{alg:expected-shellwise-search}
\begin{algorithmic}[1]
\Require costs $\vercost,\rewcost>0$ and prompt $x$
\Ensure A response $y$ with observed label $\mathsf{Ver}(x,y)=1$

\State $\cmin \gets \min\{\rewcost,\vercost\}$
\State $\Pool \gets \emptyset$ \Comment{pool of generated but unverified pairs}

\For{$s=0,1,2,\dots$}
    \State $\mathcal S_s \gets \left\{(a,b)\in \mathbb Z_{\ge 0}^2 :
    a\le b,\;
    2^s \cmin \le \rewcost\,2^b + \vercost\,2^{\,b-a} < 2^{s+1}\cmin
    \right\}$

    \If{$\mathcal S_s=\emptyset$}
        \State continue
    \EndIf

    \State $b_s^\star \gets \max\{b:(a,b)\in\mathcal S_s\}$ ,  $j_s^\star \gets \max\{b-a:(a,b)\in\mathcal S_s\}$

    \State $m_s \gets  2^{b_s^\star+1}$ , $k_s \gets 6\cdot 2^{j_s^\star}$

    \State Draw fresh responses $y_1,\dots,y_{m_s}\stackrel{\mathrm{i.i.d.}}{\sim}\pi(\cdot\mid x)$
    \State Compute reward scores $\{R_i=\mathsf{RM}(x,y_i)\}_{i \in [m_s]}$ and add $\Pool \gets \Pool \cup \{(y_i,R_i):i\in[m_s]\}$
    \State {Relabel the elements of $\Pool$ as
    $(\tilde y_1,\tilde R_1),\dots,(\tilde y_{|\Pool|},\tilde R_{|\Pool|})$
    so that
    \( \tilde R_1\ge \tilde R_2\ge \cdots\ge \tilde R_{|\Pool|} . \)}

    \For{$j=1,\dots,\min {\{k_s,|\Pool|\}}$}
        \State Query $\mathsf{Ver}(x,\tilde{y}_{(j)})$
        \State Remove $(\tilde y_j,\tilde R_j)$ from $\Pool$
        \If{$\mathsf{Ver}(x,
        \tilde{y}_{(j)})=1$}
            \State \Return $\tilde{y}_{(j)}$
        \EndIf
    \EndFor
\EndFor
\end{algorithmic}
\end{algorithm}

\begin{theorem}
\label{thm:expected-shellwise}
Under the setup of \cref{sec:problem-setup}, fix costs
\(\rewcost,\vercost>0\). Let \(\Alg_{\Adap}\) be the shellwise active search
policy in \cref{alg:expected-shellwise-search}. Then, for every feasible
prompt \(x\), every generator \(\pi(\cdot\mid x)\), and every induced pair
\((\Dist,h^\star)\) with \(h^\star\in\mathcal H_{\mathrm{non\text{-}dec}}\) (see \cref{def:active-search-instance}),
\(\Alg_{\Adap}\) is sound and satisfies
\[
    \mathbb E\!\left[J(\Alg_{\Adap};h^\star,\Dist)\right]
    \le
    154\,J^\star(h^\star,\Dist),
\]
where $J^\star(h^\star,\Dist)$ is the optimal cost of the distribution-aware policy.
\end{theorem}

\begin{remark}
The multiplicative nature of this guarantee is a critical feature of \Adap. Because the policy's cost scales proportionally with the distribution-aware optimum, ``easy'' prompts that admit a low optimal cost are answered cheaply. This directly operationalizes our core motivation that test-time compute varies significantly based on the difficulty of the specific prompt (see \cref{fig:mk-iso}).
\end{remark}

\subsection{Main ideas behind ADAP}  Fix a prompt \(x\). We will suppress the dependence of parameters on subscript \(x\) for brevity.
{The distribution-aware policy from \cref{sec:distribution-aware} shows that, when \(h^\star\) is known and non-decreasing, the right behavior is a threshold rule on the reward score (see \cref{lem:monotone-threshold-reduction}). This motivates comparing \Adap to an arbitrary threshold \(t\). Define}
$$ q_t=\Pr_{R\sim \Dist}(R\ge t) \quad\text{and}\quad s_t=\Pr(R\ge t,V=1)=\mathbb E_{R\sim \Dist}\!\left[h^\star(R)\mathbf 1\{R\ge t\}\right].$$ A threshold policy that generates candidates and verifies only
when \(R\ge t\) has expected cost
\[
    J_t=\frac{\rewcost+\vercost q_t}{s_t}.
\]
Indeed, each generated candidate costs \(\rewcost\), incurs an additional verification cost with
probability \(q_t\), and succeeds with probability \(s_t\). The unknown quantities \(q_t\) and \(s_t\) determine the relevant sample and
verification scales. If 
\[
    q_t\approx 2^{-a},
    \qquad
    s_t\approx 2^{-b},
    \qquad a\le b,
\] 
then \(2^b\) generated candidates have constant total success probability above threshold, while the expected number of candidates above threshold is $2^b q_t \approx 2^{b-a}$.
Thus the pair \((a,b)\) suggests the natural scales $m\asymp 2^b$ and $k\asymp 2^{b-a}$. The corresponding cost scale is $  B_{a,b}
    =
    \rewcost 2^b+\vercost 2^{b-a}.$

Since \Adap does not know which threshold, and hence which pair \((a,b)\), is appropriate, it searches over these dyadic possibilities by cost scale, similar to the general techniques in \citep{cesa1997use}. Shell \(s\) contains all pairs \((a,b)\) whose cost \(B_{a,b}\) is within a factor of two of \(2^s\cmin\) (see \cref{lem:expectation-dyadic}). At shell \(s\), the policy chooses the largest sample scale and the largest verification scale represented in that shell, generating \(m_s\) new candidates and allowing \(k_s\) verifier calls. Starting from small shells and moving geometrically upward ensures that the cost paid before reaching the right scale is only a constant-factor geometric prefix.

The implementation is pool-based. \Adap keeps a pool of generated but unverified responses. Each shell adds its newly generated candidates to this pool, ranks all candidates by reward, and verifies the top $k_s$. Every queried response is removed from the pool, regardless of the verifier outcome. Thus the pool always contains exactly the candidates whose labels are still unknown, and high-reward candidates discovered in earlier shells remain available to later shells if they were not yet queried.

The monotonicity assumption is what makes this ranking meaningful. Since $h^\star$ is non-decreasing, higher reward scores have weakly larger conditional probability of passing verification. Therefore, at any point, the best use of a limited verification budget is to spend it on the highest-reward unverified candidates. The proof below formalizes this intuition through a success-mass argument: at the correct dyadic scale, the pool's top-ranked candidates contain enough conditional success probability for the shell to succeed with constant probability, and later shells amplify this success probability rapidly.

\section{Experiments}
\label{sec:experiments}

\subsection{Setup}
\label{sec:exp-setup}

We evaluate \Adap \footnote{The implementation of ADAP is publicly available at \url{https://github.com/yhkalayci/efficient_query}} on two domains where verification is costly relative to sampling: mathematical reasoning and competitive programming.

\paragraph{Mathematical reasoning.}
We use the HMMT \citep{hmmt_dataset} February 2024 and 2025 contests ($60$ problems total). Solutions are sampled from \texttt{Qwen2.5-Math-7B} (base) \citep{qwen25} at temperature $0.7$, top-$p$ $0.95$, with $N=512$ samples per problem. Each candidate is scored by \texttt{Qwen2.5-Math-PRM-7B} \citep{qwen25-prm}; we use the last-step score as the per-sample reward. Verification is exact answer-string matching against the published numerical answer indicated as \texttt{boxed\{\}}. After filtering to problems with at least one correct sample, $27$ problems remain. Results with two alternative generators (\texttt{Qwen2.5-14B} \cite{qwen25} and \texttt{DeepSeek-R1-Distill-Qwen-7B} \cite{deepseek-r1}) appear in Appendix~\ref{app:other-math-models}.

\paragraph{Coding.}
We use medium- and hard-difficulty problems from LiveCodeBench (LeetCode, AtCoder, Codeforces). Solutions are sampled from \texttt{Qwen2.5-Coder-3B} (base) \citep{qwen25-coder} under identical settings. Each candidate is scored by \texttt{CodeScaler-8B} \citep{codescaler}. Verification runs the candidate against the full hidden test suite in a subprocess; a sample is correct only if every test passes. After excluding problems without any correct solution, $83$ problems remain.

\paragraph{Cost model and trial protocol.}
We set $c_{\mathrm{rew}}=1$ and $c_{\mathrm{ver}}=10$, so the cost of $N_\text{rew}$ draws and $N_\text{ver}$ verifications is $N_\text{rew} \cdot c_{\mathrm{rew}} + N_\text{ver} \cdot  c_{\mathrm{ver}}$. A trial succeeds if at least one verified sample is correct. Each (problem, policy) pair is evaluated over $10$ random permutations of sample order, with the same permutation used across policies to enable paired comparison. In Appendix~\ref{app:reward-signal}, we demonstrated that the reward model reliably ranks correct samples above incorrect ones with higher probability as we assumed in our theoretical model. To demonstrate cost-ratio robustness across $c_{\mathrm{ver}}/c_{\mathrm{rew}} \in \{1,10,20,30\}$, we presented results in Appendix~\ref{app:cost-sweep}.

\subsection{Baselines}
\label{sec:exp-algorithms}

\paragraph{Sample-aware lower bound $\SAP$.}

For each individual trial, $\SAP$ retroactively selects the cheapest $(N_{\text{rew}},N_{\text{ver}})$ such that drawing $N_{\text{rew}}$ samples and verifying the top-$N_{\text{ver}}$ by reward would yield at least one correct answer on that exact trial. The pool and ordering are fixed within a trial, so success is deterministic. $\SAP$ selects budgets, not candidates: if $j$ incorrect candidates outrank the first correct one in the drawn prefix, success requires at least $j+1$ verifications. Its cost lower-bounds successful fixed-pair policies on that trial, but choosing its budgets requires hindsight verifier labels and is not realizable. We report the mean of these per-trial minimum costs over problems and the $10$ permutations.

\vspace{-1em}
\paragraph{Difficulty-adaptive policy $\DAP_k$.}
$\DAP_k$ exploits knowledge of each problem's difficulty, measured by the empirical pass rate $\hat{r}_p$ defined as number of correct generations divided by total generations. Problems are sorted by $\hat{r}_p$ and partitioned into $k$ contiguous difficulty classes $G_1,\ldots,G_k$; within class $G_i$ the policy commits to a class-specific fixed pair $((N_{\text{rew}})_i,(N_{\text{ver}})_i)$ that achieves $100\%$ success on every problem in $G_i$ at minimum average cost. The partition and all $(((N_{\text{rew}})_i^\star,(N_{\text{ver}})_i^\star))$ are selected jointly by dynamic programming to minimize total mean cost separately for each ordering. $\DAP_k$ is an oracle baseline which requires knowing per-problem pass rates in advance. However, it models the difficulty-aware strategies commonly used in practice.

The degenerate case $k=P$, where $P$ is the number of problems, renders $\DAP_P$ equivalent to $\SAP$. In Appendix~\ref{app:dap-robust}, we introduce a weaker variant, $\DAP_k^{\mathrm{robust}}$, which requires a single partition and one budget pair per class that ensure success across all $10$ permutations. In the special case $k=1$, $\DAP_1^{\mathrm{robust}}$ applies the same pair $(N_{\text{rew}},N_{\text{ver}})$ uniformly to all problems. $\Adap$ is more competitive against this variant: at the default cost ratio, it is cheaper for every reported $k$ on coding and for $k\in\{1,2,3,5\}$ on math. Both variants use verifier labels and are not implementable.

\vspace{-0.5em}
\subsection{Results}
\vspace{-0.5em}
Table~\ref{tab:summary} reports mean generation count $\overline{N_{\text{rew}}}$, verification count $\overline{N_{\text{ver}}}$, mean cost, success rate, and cost ratio relative to $\Adap$ for all policies on both tasks.

\begin{figure}[t]
    \centering
    \begin{subfigure}[b]{0.495\linewidth}
        \includegraphics[width=\linewidth]{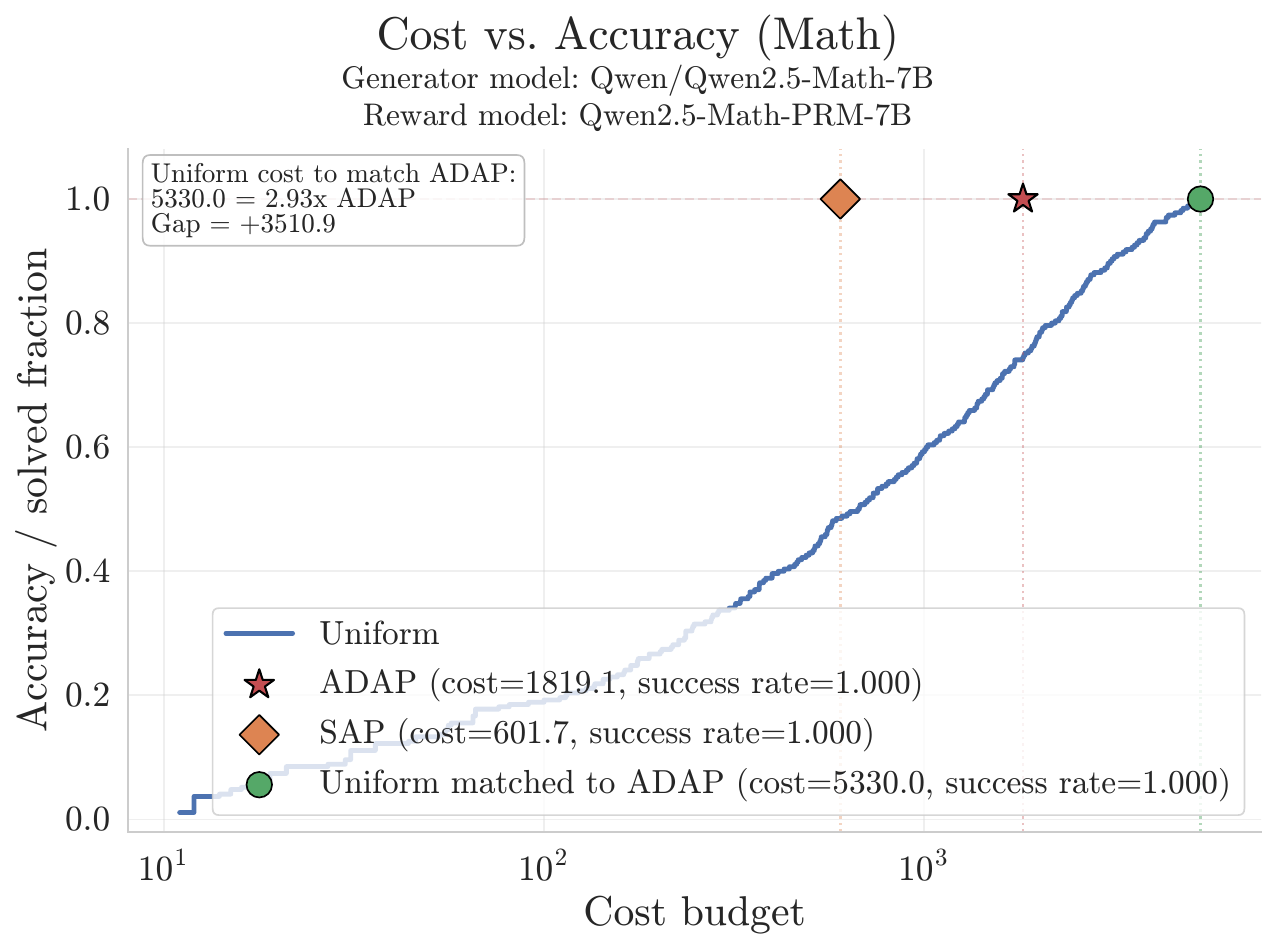}
        \caption{Math: Cost vs.\ Accuracy.}
    \end{subfigure}
    \hfill
    \begin{subfigure}[b]{0.495\linewidth}
        \includegraphics[width=\linewidth]{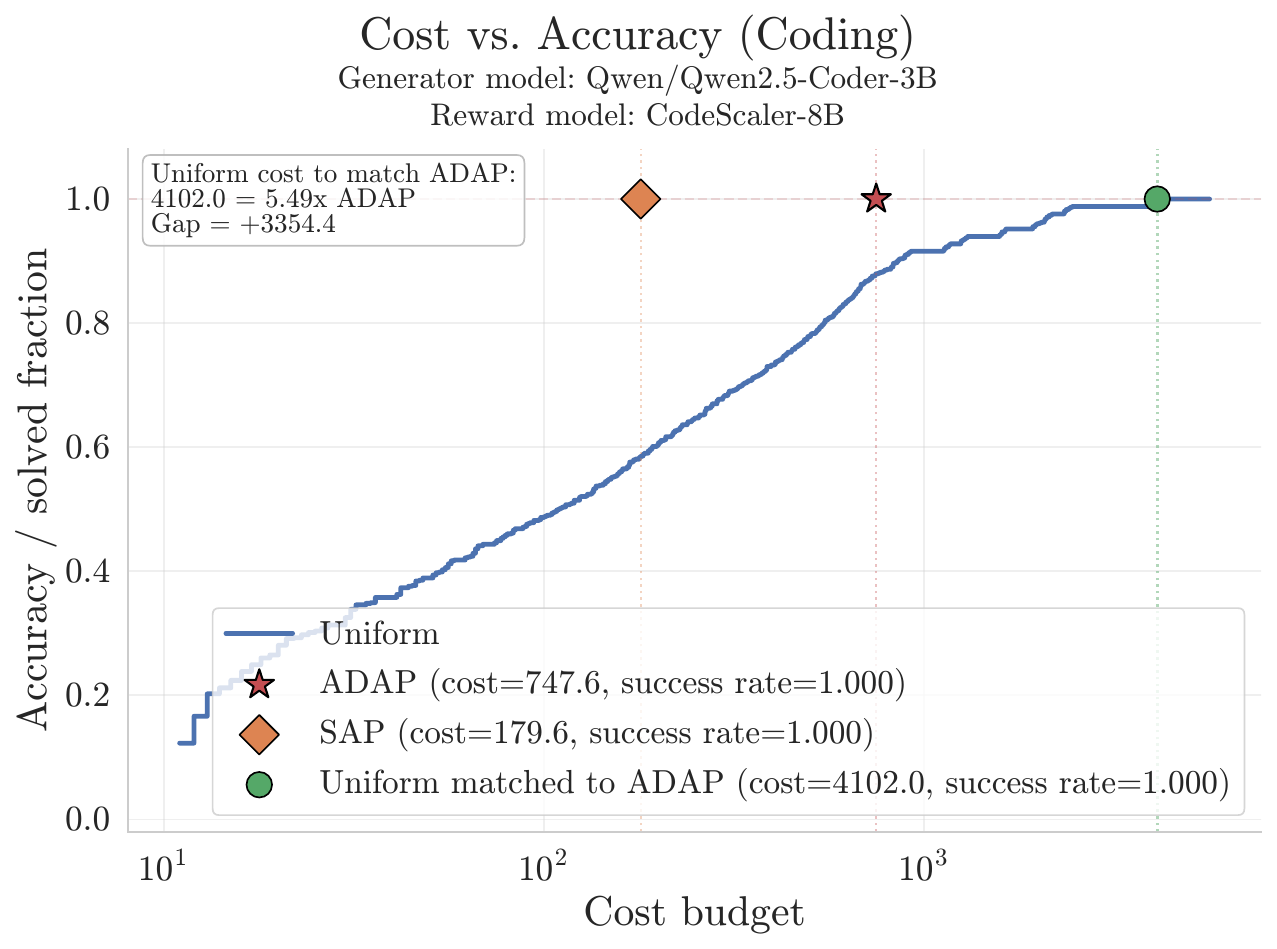}
        \caption{Coding: Cost vs.\ Accuracy.}
    \end{subfigure}\\[0.3em]
    \begin{subfigure}[b]{0.495\linewidth}
        \includegraphics[width=\linewidth]{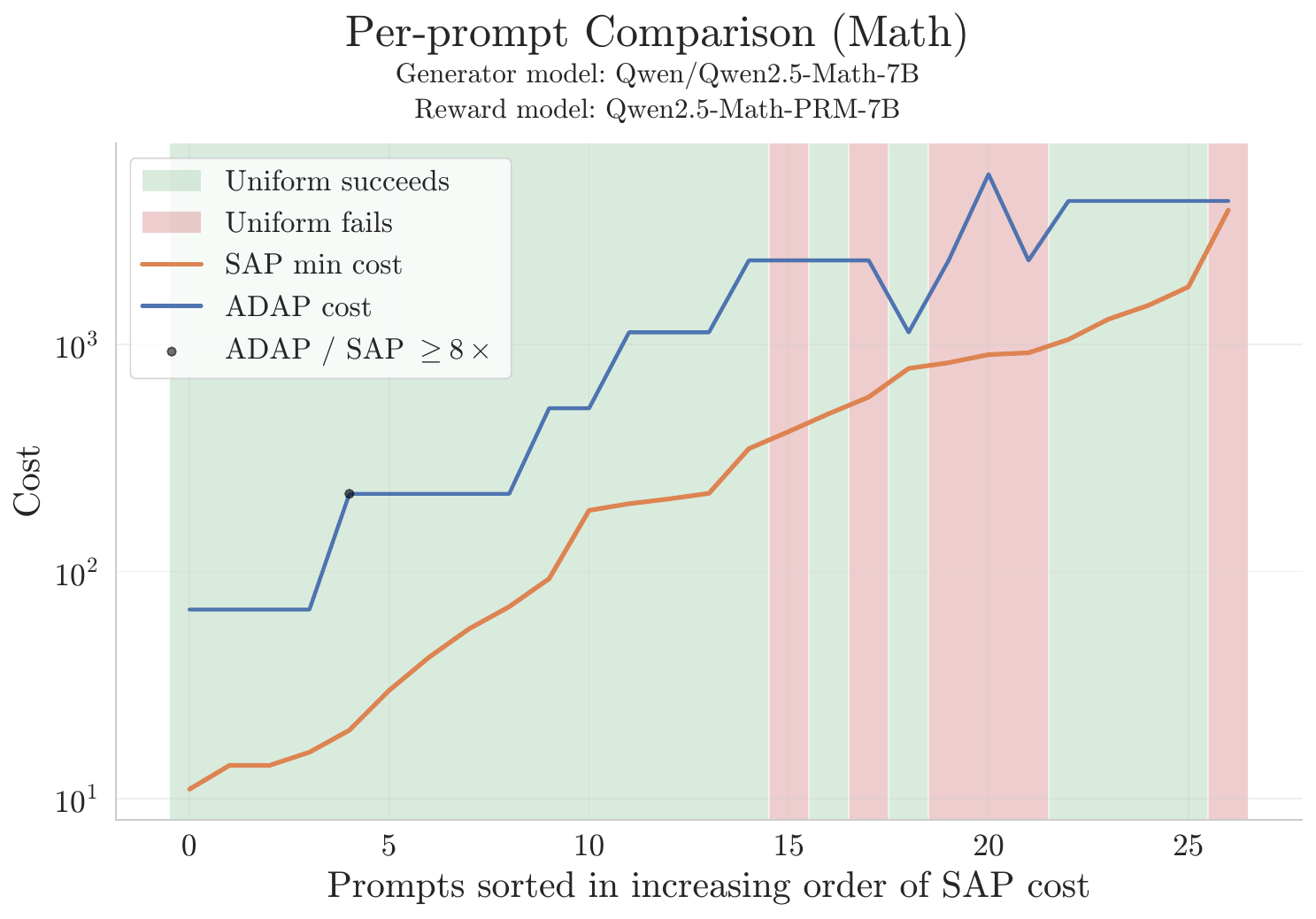}
        \caption{Math: Per-prompt comparison across three policies.}
    \end{subfigure}
    \hfill
    \begin{subfigure}[b]{0.495\linewidth}
        \includegraphics[width=\linewidth]{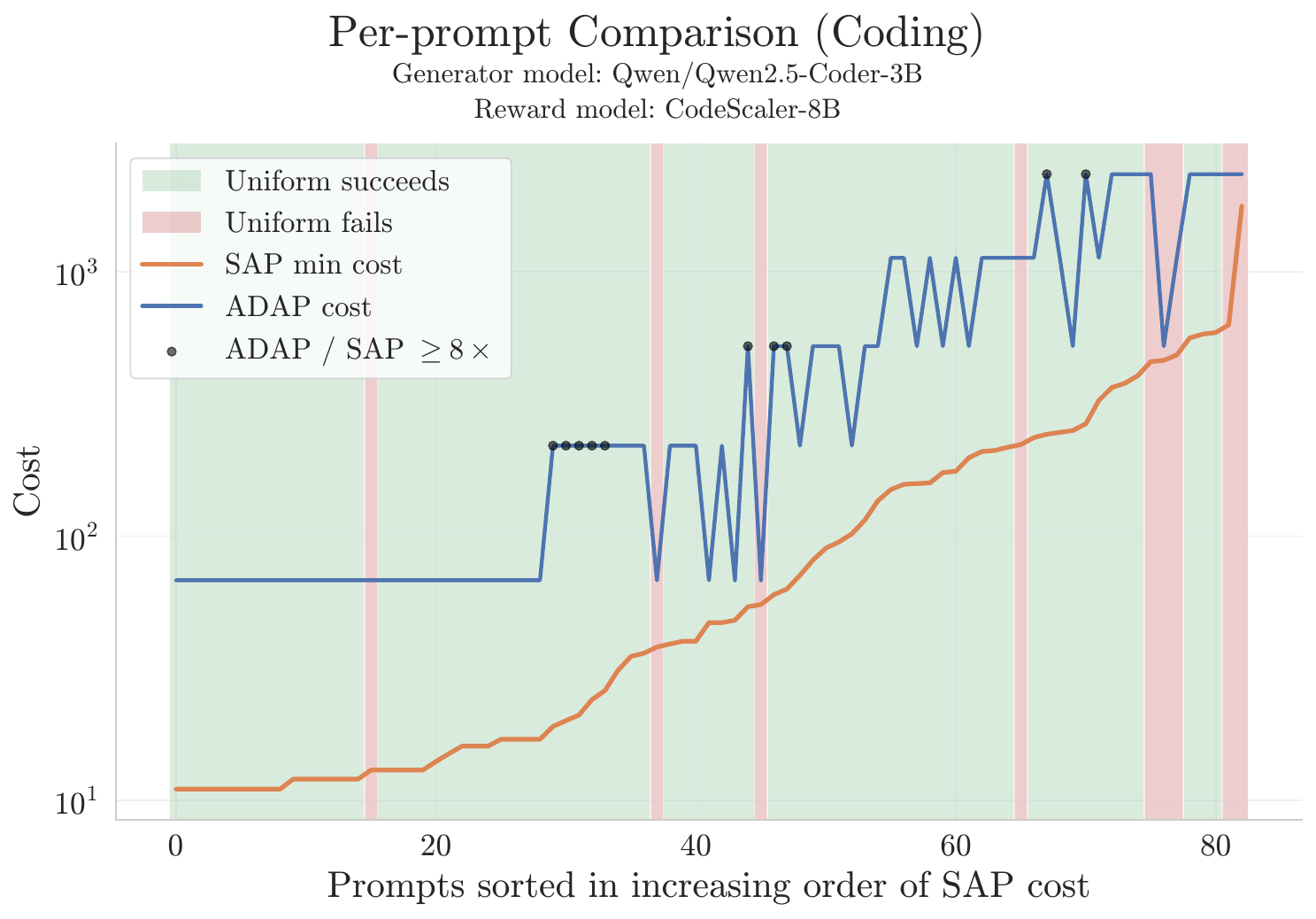}
        \caption{Coding: Per-prompt comparison across three policies.}
    \end{subfigure}
    \caption{\textbf{Top:} success rate vs.\ cost budget for the best Uniform strategy (blue step curve), with \SAP (orange diamond) and \Adap (red star) marked at their mean costs. \textbf{Bottom:} per-prompt costs of \Adap (blue) and \SAP (orange), sorted by \SAP cost; background shading indicates whether the cost-matched Uniform strategy succeeds (green) or fails (red) on each trial.}
    \label{fig:results}
\end{figure}

\begin{table}[h]
    \small
    \renewcommand{\arraystretch}{0.6}
    \caption{Summary of all policies on Math (HMMT, $27$ problems, $10$ permutations) and Coding (LiveCodeBench, $83$ problems, $10$ permutations). Cost ratio is relative to $\Adap$. $\DAP_k$ rows use the stronger variant; results for the weaker variant appear in Appendix~\ref{app:ablation}. The reported $\overline{N_{\mathrm{rew}}}$ and $\overline{N_{\mathrm{ver}}}$ are mean generation and verification counts over all problem--permutation trials.}
    \label{tab:summary}
    \centering
    \begin{tabular}{lrrrrrrrrrr}
    \toprule
    & \multicolumn{5}{c}{Math (27 solvable prompts)} & \multicolumn{5}{c}{Coding (83 solvable prompts)} \\
    \cmidrule(lr){2-6}\cmidrule(lr){7-11}
    Policy & $\overline{N_{\rm rew}}$ & $\overline{N_{\rm ver}}$ & Cost & Ratio & Succ. & $\overline{N_{\rm rew}}$ & $\overline{N_{\rm ver}}$ & Cost & Ratio & Succ. \\
    \midrule
    $\SAP$ & 140.7 & 46.1 & 601.7 & $\times0.33$ & 1.000 & 92.4 & 8.7 & 179.6 & $\times0.24$ & 1.000 \\
    $\Adap$ & 273.5 & 154.6 & 1819.1 & $\times1.00$ & 1.000 & 140.4 & 60.7 & 747.6 & $\times1.00$ & 1.000 \\
    $\Uni_{C_{\Adap}}$ & 266.0 & 155.0 & 1816.0 & $\times1.00$ & 0.744 & 426.0 & 32.0 & 746.0 & $\times1.00$ & 0.880 \\
    \midrule
    $\DAP_{1}$ & 438.8 & 415.8 & 4596.8 & $\times2.53$ & 1.000 & 472.8 & 331.5 & 3787.8 & $\times5.07$ & 1.000 \\
    $\DAP_{2}$ & 317.5 & 248.5 & 2802.4 & $\times1.54$ & 1.000 & 305.1 & 138.7 & 1692.6 & $\times2.26$ & 1.000 \\
    $\DAP_{3}$ & 268.0 & 190.7 & 2175.5 & $\times1.20$ & 1.000 & 260.0 & 92.4 & 1183.9 & $\times1.58$ & 1.000 \\
    $\DAP_{5}$ & 257.9 & 132.6 & 1583.6 & $\times0.87$ & 1.000 & 208.9 & 59.0 & 799.0 & $\times1.07$ & 1.000 \\
    $\DAP_{10}$ & 208.3 & 72.5 & 933.5 & $\times0.51$ & 1.000 & 170.5 & 30.3 & 473.3 & $\times0.63$ & 1.000 \\
    \bottomrule
    \end{tabular}

\end{table}

\paragraph{Math (HMMT).}
$\Adap$ achieves $100\%$ success at mean cost $1819.1$, while $\SAP$'s mean cost is $601.7$ ($0.33\times$). As shown in Figure~\ref{fig:results}, $\Adap$'s per-trial cost tracks $\SAP$'s closely across the difficulty spectrum. At the same budget, $\Uni_{C_{\Adap}}$ succeeds on only $74.4\%$ of trials, a notable gap that confirms a uniform fixed strategy cannot replicate $\Adap$'s reliability at equal cost. Among difficulty-stratified baselines, $\DAP_1$ costs $2.53\times$ more than $\Adap$, and the gap narrows as $k$ grows: $\DAP_5$ costs $0.87\times$ as much as $\Adap$, while $\DAP_{10}$ reaches $0.51\times$. With only $27$ problems, ten difficulty classes place on average just two problems per class, making $\DAP_{10}$ nearly a per-problem oracle and thus approaching $\SAP$.

\paragraph{Coding (LiveCodeBench).}
$\Adap$ achieves $100\%$ success at mean cost $747.6$, with $\SAP$ costing $179.6$ ($0.24\times$). Figure~\ref{fig:results} again shows $\Adap$'s per-trial cost following $\SAP$'s closely. $\Uni_{C_{\Adap}}$ falls short at $88.0\%$ success, a clear failure mode of uniform strategies on a heterogeneous task.
$\DAP_k$ costs $5.07\times$ as much as $\Adap$ at $k=1$, $1.58\times$ at $k=3$, and $0.63\times$ at $k=10$; its cost decreases to $\SAP$ at $k=83$.

\section{Learning-Theoretic View of Active Search}
\label{sec:seperation}

In \cref{sec:non-dec-alg}, monotonicity of the score-conditioned success function gave \Adap a constant-factor guarantee relative to the distribution-aware optimum. We now show that some such structure is necessary: for unstructured $\mathcal{H}$, no sound active-search policy can achieve a uniform constant-factor guarantee. The obstruction is that a positive response may be hidden among many locations that are indistinguishable without verification, while a distribution-aware oracle can focus on the right one immediately. We formalize this using the \emph{centered star number}, a variant of the classical star number of \citep{balcan2010true,hanneke2015minimax} adapted to finding a single positive witness. We first prove the lower bound in the binary setting, and then extend it to probabilistic concept classes in \cref{sec:start-num-probalistic}. We also show the seperation and connections between active search and active learning \citep{balcan2006agnostic,balcan2016active} in \cref{sec:activelearn-search}.

\begin{definition} \label{def:starnum}
Let $\mathcal{Z}$ be an arbitrary input space. Let $\conceptclass \subseteq \{0,1\}^\mathcal{Z}$ be a binary concept class over $\mathcal{Z}$ such that it includes the all zero function which is denoted by $h_0$. Then, define 
\[
\starnum_0(\conceptclass) \;=\; \max\Big\{\, m \in \mathbb{N} \;:\; \exists\, z_1, \dots, z_m \in \mathcal{Z},\; h_1, \dots, h_m \in \conceptclass \text{ with } h_i(z_j) = \indic{i = j} \,\Big\}.
\]
with $\starnum_0(\conceptclass) = \infty$ if no finite maximum exists.
\end{definition}

Equipped with this definition, we state an impossibility result.

\begin{theorem}
\label{thm:starnum}
Let $\mathcal{Z}$ be an arbitrary input space. Fix costs \(\rewcost,\vercost>0\). Let
\(\conceptclass\subseteq\{0,1\}^{\mathcal{Z}}\) be a class
containing the all-zero function, and let \(\starnum_0(\conceptclass)\) be its centered star
number. For every sound active search policy
\(\Alg\) from \cref{def:active-search}, there exists \(h^\star\in\conceptclass\) and
\(\Dist\in\ProbMeasures{\mathcal{Z}}\) such that 
$$\displaystyle
    {\mathbb E[J(\Alg;h^\star,\Dist)]}
    \ge
   \frac{1}{4}
    \min\left\{
        \starnum_0(\conceptclass),\frac{\vercost}{\rewcost}
    \right\} J^\star(h^\star,\Dist),
$$
where $J^\star(h^\star,\Dist)$ is the distribution-aware optimum cost.
\end{theorem}

Next we discuss the implication of this result: If \(\starnum_0=\infty\), then no sound active search policy can have
a uniform constant-factor guarantee against the distribution-aware oracle. In
particular, in the regime \(\vercost\gg \rewcost\), the worst-case ratio can be
as large as \(\Omega(\vercost/\rewcost)\). This result shows that without assuming structure on $h$, which corresponds to the case that \(\starnum_0=\infty\), it is impossible to develop a policy that satisfies a guarantee analogous to \cref{thm:expected-shellwise}.

\subsection{Matching Upper Bound}

In this part, we complement our lower bound in \cref{thm:starnum} with an algorithm, \(\Alg_{\rm CS}\) in \cref{alg:centered-star}, that achieves a competitive ratio $O\left( \min\left\{
        \starnum_0(\conceptclass),\frac{\vercost}{\rewcost}
    \right\}\right)$ compared to the distribution-aware benchmark. Note we do
not claim \(\Alg_{\rm CS}\) is computationally efficient for arbitrary \(\mathcal H\).

\begin{theorem}
\label{thm:centered-star-upper}
Let \(\conceptclass\subseteq\{0,1\}^{\mathcal Z}\) contain the all zero function. Assume \(\vercost\ge\rewcost\). Then for every feasible binary
instance \((h^\star,\Dist)\) with \(h^\star\in\conceptclass\), \(\Alg_{\rm CS}\) in \cref{alg:centered-star} is sound and satisfies
\[
    \EE[J(\Alg_{\rm CS};h^\star,\Dist)]
    \le
    6\min\left\{\starnum_0(\conceptclass),\frac{\vercost}{\rewcost}\right\}
    J^\star(h^\star,\Dist).
\]
\end{theorem}

We first explain the main observation that leads to the matching upper bound. The idea is to use the centered star number to certify an entire generated batch with only a few verifier calls. Given a finite batch \(S \in \mathcal{Z}^{\star}\), each \(h\in\conceptclass\) induces a positive trace
\(A_h(S):=\{z\in S:h(z)=1\}\). A hitting set \(Q\subseteq S\) for the nonempty traces has the property that, if the true target has any positive point in \(S\), then it has one in \(Q\). Thus verifying \(Q\) either finds a positive or certifies that the whole batch can be discarded. The centered star number bounds the size of such certificates, as shown by \cref{lem:centered-star-hitting-set}.

\begin{algorithm}[t]
\caption{\(\Alg_{\rm CS}\): Centered-Star Active Search}
\label{alg:centered-star}
\begin{algorithmic}[1]
\Require concept class \(\conceptclass \in \{0,1\}^\mathcal{Z}\), costs \(\rewcost,\vercost>0\)
\Ensure a generated point with observed verifier label \(1\)
\If{ \(\starnum_0> \frac{\vercost}{\rewcost}\)}
    \While{true}
        \State Generate \(Z\sim\Dist\) and pay \(\rewcost\) and  Verify \(Z\) and pay \(\vercost\)
        \If{the verifier returns \(1\)}
            \State \Return \(Z\)
        \EndIf
    \EndWhile
\Else
    \State \(n\gets \lceil \frac{\vercost}{\rewcost} \starnum_0\rceil\)
    \While{true}
        \State Generate \((Z_1,\dots,Z_n)\sim\Dist^{\otimes n}\) and pay \(n\rewcost\)
        \State Let \(S=(Z_1,\dots,Z_n)\)
        \State For every \(h\in\conceptclass\), construct \(A_h(S):=\{i\in[n]:h(Z_i)=1\}\)
        \State Choose a \emph{minimum-cardinality hitting set} \(I\subseteq[n]\) for
        \(\{A_h(S):h\in\conceptclass,\ A_h(S)\neq\emptyset\}\)
        \For{\(i\in I\)} \Comment{see \cref{lem:centered-star-hitting-set}}
            \State Verify \(Z_i\) and pay \(\vercost\)
            \If{the verifier returns \(1\)}
                \State \Return \(Z_i\)
            \EndIf
        \EndFor
    \EndWhile
\EndIf
\end{algorithmic}
\end{algorithm}

\begin{remark}
The characterization of the active search in the binary case  in \cref{thm:starnum,thm:centered-star-upper} is based on a centered version of the usual star number. The usual star number, that governs distribution-free \emph{active learning} \citep{hanneke2014theory,hanneke2015minimax}, allows the center concept to be arbitrary. Our \(\starnum_0\) fixes the center to be the all-zero concept, so
$
    \starnum_0(\conceptclass)\le \starnum(\conceptclass).
$
This is the right obstruction for active search since the learner does not need to identify the whole target concept, only to find one verified positive. Thus \(\starnum_0\) captures how many different locations can hide the unique positive against an all-negative baseline. Active learning can be harder because it must distinguish all nearby alternatives around arbitrary centers. In \cref{sec:activelearn-search}, we formally show per-instance active searching is easier than active learning.
\end{remark}

\section{Conclusion and Limitations}
We formulated generate--rank--verify inference as a cost-sensitive active search problem with cheap reward scores and costly final verification, and proposed \Adap, an online policy that adapts generation and verification effort per prompt.
Under monotone score-conditioned success probabilities, \Adap achieves a constant-factor guarantee relative to the distribution-aware benchmark; our lower bound shows that some such structure is necessary.
The work rests on several simplifying assumptions: the guarantee requires monotonicity, the evaluation covers only feasible prompts, costs are treated as fixed scalars rather than varying with response length, and decisions are based on reward scores alone while in practice verification can exploit response content directly (e.g., detecting syntactic invalidity without invoking a costly checker).
Extensions to variable costs, multiple reward models, parallel generation/verification, and richer inductive biases as well as determining the optimal universal approximation constant under monotonicity are natural future directions.

\begin{ack}
The authors thank Vincent Cohen-Addad for inspiring discussions in the early stages of this project, and Avrim Blum for pointing out connections to the active learning literature. The authors acknowledge the Center for Advanced Research Computing (CARC) at the University of Southern California for providing computing resources that have contributed to the research results reported within this publication. URL: \url{https://carc.usc.edu}.

Shaddin Dughmi and Yusuf Hakan Kalayci were supported by the Air Force Office of Scientific Research under award number FA9550-24-1-0261. This work was done while Shaddin Dughmi was on sabbatical as the Carter and Tania Neild visiting professor at Northwestern University, as well as a visiting professor in the Data Science Institute at the University of Chicago. Mahdi Haghifam's work was conducted while he was visiting the Simons Institute for the Theory of Computing. Yusuf Hakan Kalayci's work was conducted while he was a PhD student at the University of Southern California.

\end{ack}

\printbibliography

\newpage

\appendix

\section*{Contents}
\setcounter{tocdepth}{2}   %
\makeatletter
\@starttoc{toc}
\makeatother

\newpage

\section{Related Work}
\label{sec:related-work}
\paragraph{Generate--rank--verify foundations.}
Many inference-time reasoning systems follow a generate--evaluate--select pattern.
In math reasoning, trained outcome verifiers rank complete solutions~\citep{cobbe2021training}, while process reward models and generative verifiers score intermediate or complete reasoning traces~\citep{lightman2023verify,wang2024mathshepherd,setlur2025processverifiers,zhang2025genrm,khalifa2025thinkprm,zhang2025lessonsprm}.
In code generation, AlphaCode used large-scale sampling followed by execution-based filtering, clustering, and selection~\citep{li2022alphacode}; CodeT and LEVER use generated tests or execution-aware learned verifiers to rank programs~\citep{chen2022codet,ni2023lever}; and CodeRM studies dynamic scaling of generated unit tests to improve reward quality~\citep{ma2025coderm}.
These works motivate our generate--rank--verify abstraction, but mainly study verifier design, empirical scaling, or engineered selection pipelines.
\vspace{-1em}
\paragraph{Adaptive test-time compute allocation.}
A line of work studies how to spend inference-time compute adaptively.
A common theme is that uniform Best-of-$N$ sampling is inefficient because the value of additional samples depends on input difficulty~\citep{snell2024scaling}.
Some methods allocate compute across inputs using predicted reward curves~\citep{damani2025hardthink}, constrained optimization under an average budget~\citep{zhai2026adaptive}, or prompt-level Best-of-$N$ budget allocation~\citep{raman2025adabon}.
Others adapt computation within a single input, stopping generation when further samples are unlikely to improve reward~\citep{kalayci2025optimalstopping,wan2025beacon}, or when answer agreement is sufficiently high in self-consistency sampling~\citep{huang2026optimalbayesian}.
Our focus is different: the policy must decide how to trade off generating additional against verifying available candidates.
\vspace{-1em}
\paragraph{Selective and imperfect verification.}
A related line studies how to use imperfect or costly verification signals.
Weak--strong verification policies decide when a cheap noisy verifier is sufficient and when to defer to a stronger verifier~\citep{kiyani2026weakstrong}, while Derailer--Rerailer uses a lightweight stability check to trigger more expensive reasoning only when needed~\citep{wan2025derailer}.
Closest to our setting, \citet{qu2026adaptive} studies a verification-cost-limited regime and allocates verifier calls across intermediate reasoning states using feasibility gates, learned pre-verification scores, and uncertainty.
In contrast, our verifier is treated as the final certification mechanism, and the main question is how costly certification should be interleaved with generation and reward scoring. 
\vspace{-0.75em}
\paragraph{Active search.}

Active search studies discovery-oriented label querying, where the goal is to find
positives rather than to learn a globally accurate classifier. Classical Bayesian active search is usually pool-based. A fixed unlabeled pool is given in advance, and a known probabilistic model or Bayesian posterior over labels guides which points to query under a labeling budget
\citep{garnett2012bayesian,jiang2017efficient,jiang2018batch}. The closest formulation to ours is cost-effective active search~\citep{jiang2019cost}, which
minimizes labeling cost to find a prescribed number of positives from a fixed pool. Our setting keeps the discovery objective of active search, but the candidate pool is generated online by sampling from a language model and scoring samples with a reward model.
Because the score distribution and reward--verifier relationship are unknown, the policy cannot rely on a known posterior over a fixed pool.
It must decide both which candidates to verify and when to generate more scored candidates, creating a generation--verification tradeoff absent from standard fixed-pool active search.

\paragraph{Learning-Theoretic View of the Challenges in LLMs.}

A related line of work studies the learnability of verifiers when we have chain-of-thought traces. \citep{balcan2025learning}
introduce PAC-style models for learning Chain-of-Thought verifiers, including simple verifiers that
generalize on traces from a fixed distribution and stronger trustable verifiers that aim to reject faulty
reasoning traces more robustly. \citep{balcan2026online} develop an online learning
framework for Chain-of-Thought verification, emphasizing the asymmetric costs of soundness errors
(accepting faulty reasoning) and completeness errors (rejecting correct reasoning), and characterize
the optimal mistake tradeoffs using soundness-completeness variants of the Littlestone dimension. Our work addresses a complementary inference-time decision problem. Rather than learning the verifier, we assume access to a final certifying verifier, such as exact answer checking,
hidden-test execution, or proof verification, and treat each verifier call as costly. We also
assume access to a cheap but noisy reward model that can rank candidate responses. The
goal is therefore not to learn a globally reliable verifier, but to adaptively decide how many
responses to generate and which high-reward candidates to verify in order to find a
verified-correct response at minimum cost.

\paragraph{Inference with Imperfect Proxies.}
Recent theoretical works analyze inference guided by imperfect reward models. \citep{huang2025best} quantify proxy imperfection via squared error relative to a hypothetical true reward, showing that Best-of-$N$ sampling overexploits tail errors; crucially, their setting relies entirely on the proxy and lacks a final verifier. \citep{rohatgi2025taming} formalize process verifiers as approximate value functions, demonstrating how local evaluation errors compound over long generation horizons. We capture proxy imperfection through a fundamentally different lens. Rather than assuming the proxy is calibrated to a true reward or tracks a process value function, we treat it strictly as a cheap, uncalibrated ranking signal. For each prompt $x$, the generator and proxy induce an unknown score distribution $\mathcal{D}_x$, while an exact but costly final verifier defines an unknown conditional success probability $h_x^\star(r)=\Pr(V_x=1\mid R_x=r)$. This decouples the noisy proxy from the ground-truth verifier, shifting the theoretical focus from bounding proxy approximation errors to cost-sensitive active search.

\section{Additional Experiments} \label{sec:appx-expriments}

\subsection{Computational Resources} 
\label{app:compute}

All experiments run on a single node with two NVIDIA L40S GPUs (48\,GB each). Generation and math reward scoring use vLLM \cite{kwon2023efficient} with TP$=$2 (Tensor Parallelism); coding reward scoring uses CodeScaler-8B via PyTorch on a single GPU; coding verification is CPU-only with 64 parallel workers and a 15-second per-test timeout. All figures below are estimates; exact per-stage timing was not logged.

\begin{table}[h]
\centering\small
\begin{tabular}{lp{6cm}l}
\toprule
Stage & Details & Estimated cost \\
\midrule
Math generation & $3$ models $\times$ $60$ problems $\times$ $512$ samples; $8$--$10$\,h/model on $2\times$L40S & $\sim$48--60 GPU-hour \\
Math reward scoring & Qwen2.5-Math-PRM-7B via vLLM, $92{,}160$ samples & $\sim$4--6 GPU-hour \\
Coding generation & Qwen2.5-Coder-3B, TP$=$2, $329\times512=168{,}448$ samples, ${\sim}750$ tok/sample & $\sim$26--40 GPU-hour \\
Coding reward scoring & CodeScaler-8B, single GPU, $42{,}496$ samples, batch size 32 & $\sim$2--3 GPU-hour \\
\midrule
\textbf{Total GPU} & & $\mathbf{\sim}$\textbf{85--110 GPU-hour} \\ 
\midrule
Coding verification & $42{,}496$ programs, 64 workers, ${\sim}20$--$30$\,CPU-s/sample (wall time ${\approx}$4--6\,h) & $\sim$250--400 CPU-hour \\
\bottomrule
\end{tabular}
\end{table}

\subsection{Reward Signal Validation}
\label{app:reward-signal}

The adaptive policy assumes that the reward model is more likely to assign higher scores to correct samples than to incorrect ones. We verify this directly for coding; math diagnostics for alternative generators appear in Appendix~\ref{app:other-math-models}.

\paragraph{Rank vs.\ correctness.}
For each problem we sort samples in decreasing reward order and record correctness at each rank position. Figure~\ref{fig:rank-correct} plots the empirical correctness probability averaged across the $83$ coding problems and Figure~\ref{fig:rank-correct-math} plots the same result averaged across the $27$ math problems. The curve is monotonically decreasing: the top-ranked sample is several times more likely to be correct than a uniformly random sample. Our policy requires only this rank monotonicity, not calibrated reward probabilities.

\paragraph{Top-$k$ coverage.}
Figure~\ref{fig:topk} plots, for each $k$, the probability that the top-$k$-by-reward set contains at least one correct sample, alongside the random-$k$ baseline (hypergeometric distribution). Using the reward model substantially raises coverage even at $k=1$, and the gap remains positive throughout the small-$k$ regime where the adaptive policy operates.

\begin{figure}[ht]
    \centering
    \begin{subfigure}[b]{0.495\linewidth}
        \includegraphics[width=\linewidth]{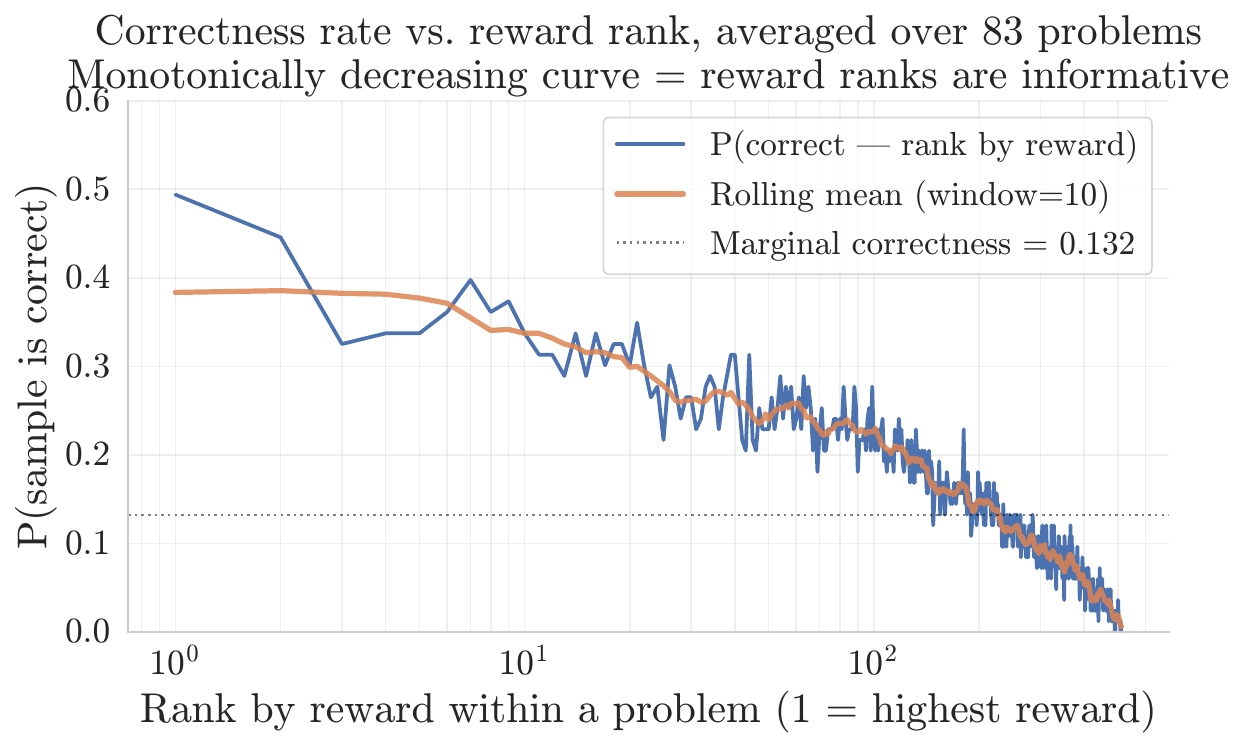}
        \caption{Correctness probability at each reward rank.}
        \label{fig:rank-correct}
    \end{subfigure}
    \hfill
    \begin{subfigure}[b]{0.495\linewidth}
        \includegraphics[width=\linewidth]{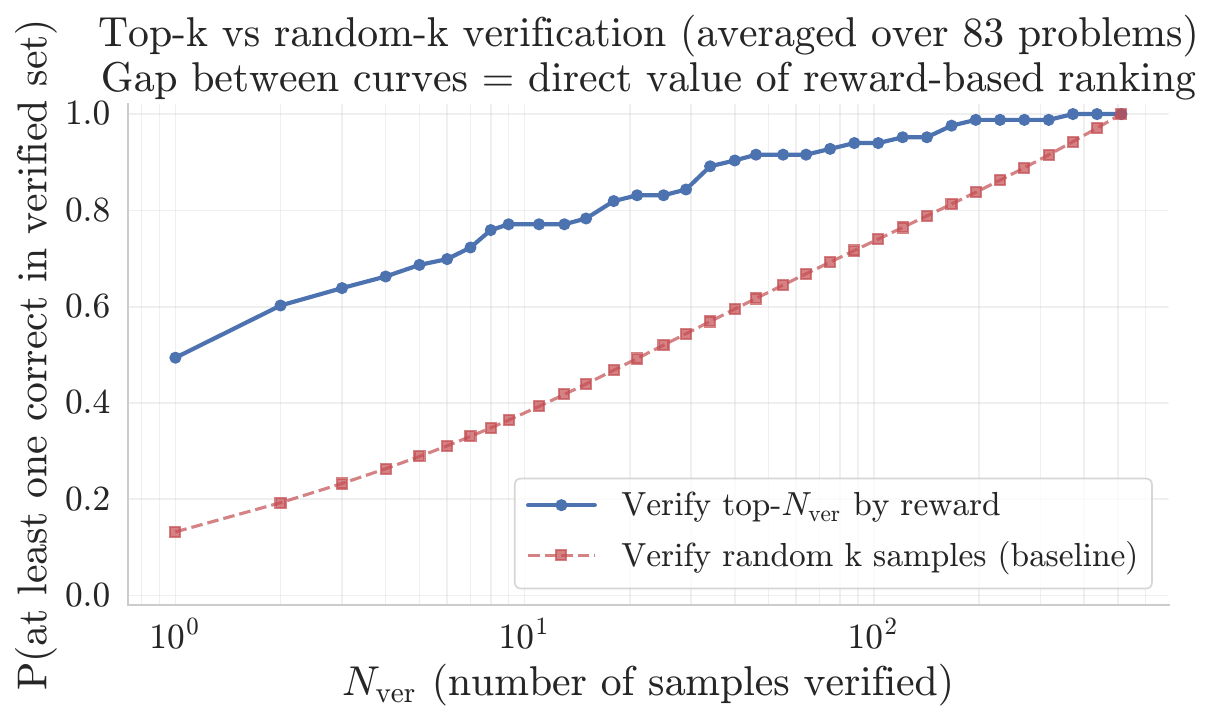}
        \caption{Top-$k$ vs.\ random-$k$ coverage.}
        \label{fig:topk}
    \end{subfigure}
    \caption{Reward signal validation on coding. Left: empirical correctness probability at each reward rank, averaged over $83$ problems (orange: rolling mean). Right: probability that the top-$k$ by reward contains a correct sample (blue) vs.\ $k$ uniform random samples (red dashed).}
    \label{fig:reward-signal}
\end{figure}

\begin{figure}[ht]
    \centering
    \begin{subfigure}[b]{0.495\linewidth}
        \includegraphics[width=\linewidth]{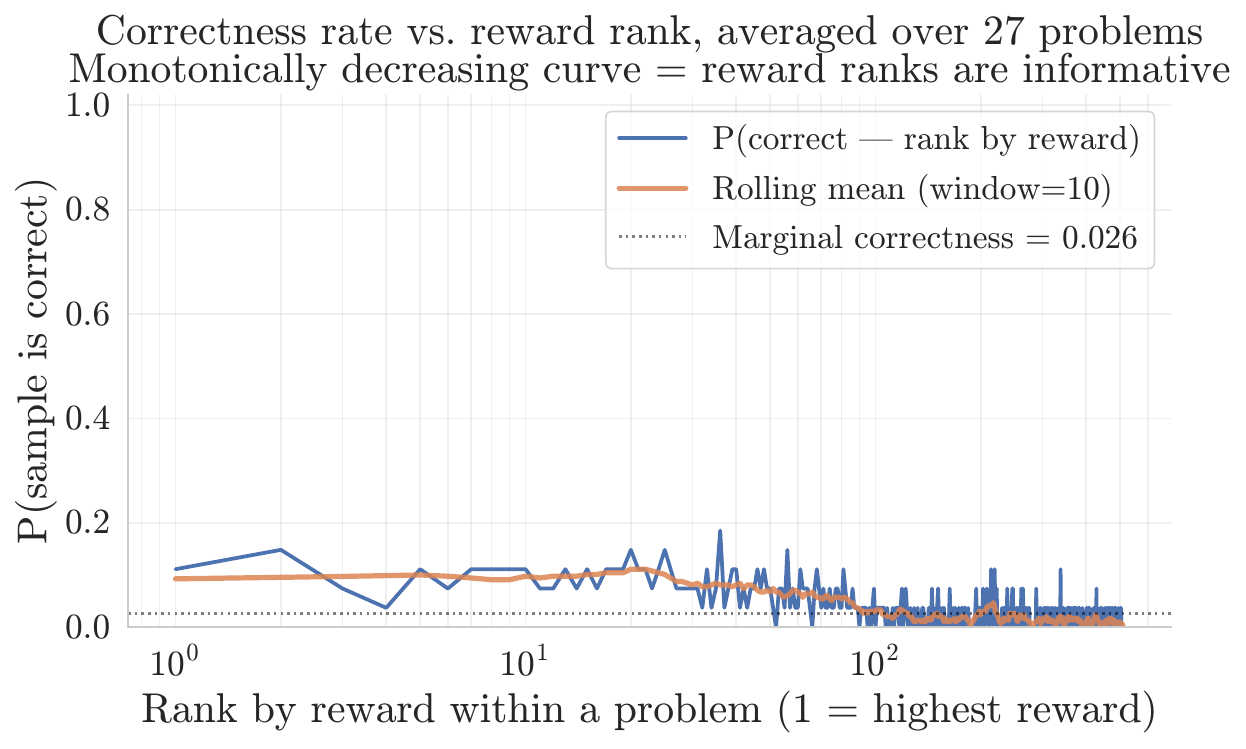}
        \caption{Correctness probability at each reward rank.}
        \label{fig:rank-correct-math}
    \end{subfigure}
    \hfill
    \begin{subfigure}[b]{0.495\linewidth}
        \includegraphics[width=\linewidth]{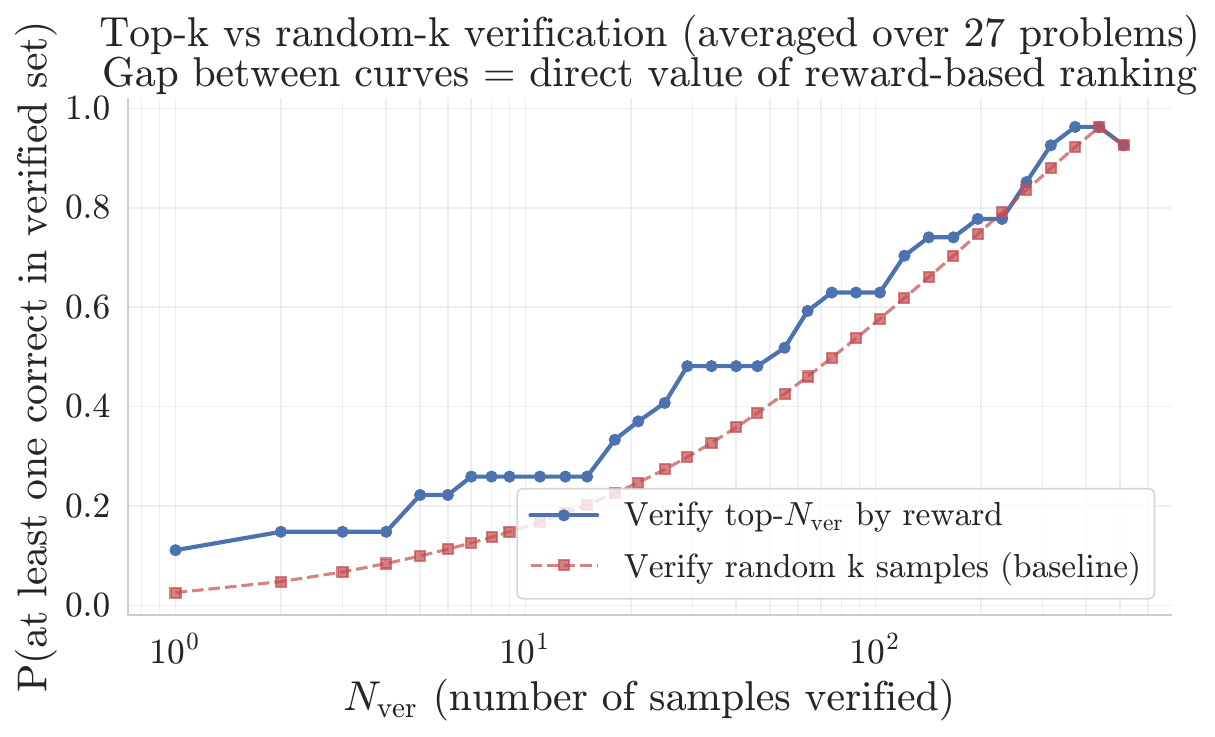}
        \caption{Top-$k$ vs.\ random-$k$ coverage.}
        \label{fig:topk-math}
    \end{subfigure}
    \caption{Reward signal validation on math. Left: empirical correctness probability at each reward rank, averaged over $27$ problems (orange: rolling mean). Right: probability that the top-$k$ by reward contains a correct sample (blue) vs.\ $k$ uniform random samples (red dashed).}
    \label{fig:reward-signal-math}
\end{figure}

\paragraph{Per-problem AUC distribution.}
To assess whether the reward signal is consistent across problems, we compute a per-problem AUC:
\[
  \mathrm{AUC}_p = \Pr\!\bigl(R > R' \;\big|\; V=1,\, V'=0\bigr)+\tfrac12\Pr\!\bigl(R=R' \;\big|\; V=1,\, V'=0\bigr),
\]
where $(R,V)$ and $(R',V')$ are independent draws from the joint distribution of prompt $p$ (i.e.\ $R=\mathsf{RM}(p,Y)$, $V=\mathsf{Ver}(p,Y)$ with $Y\sim\pi(\cdot\mid p)$, and likewise for the primed copy).

For \textbf{coding}, the distribution is strongly right-skewed (mean: $0.865$, median: $0.915$, IQR: $[0.779, 0.980]$, $96\%$ above chance), confirming that CodeScaler-8B reliably ranks correct samples higher on the vast majority of problems.

For \textbf{math}, the signal is weaker (mean: $0.597$, median: $0.609$, IQR: $[0.443, 0.785]$, $70\%$ above chance), yet the mean AUC remains meaningfully above chance and the top-$k$ coverage curve (Figure~\ref{fig:reward-signal-math}) still lies above the random baseline throughout, confirming that even a modest reward signal suffices to guide verification.

\begin{figure}[ht]
    \centering
    \begin{subfigure}[b]{0.495\linewidth}
        \includegraphics[width=\linewidth]{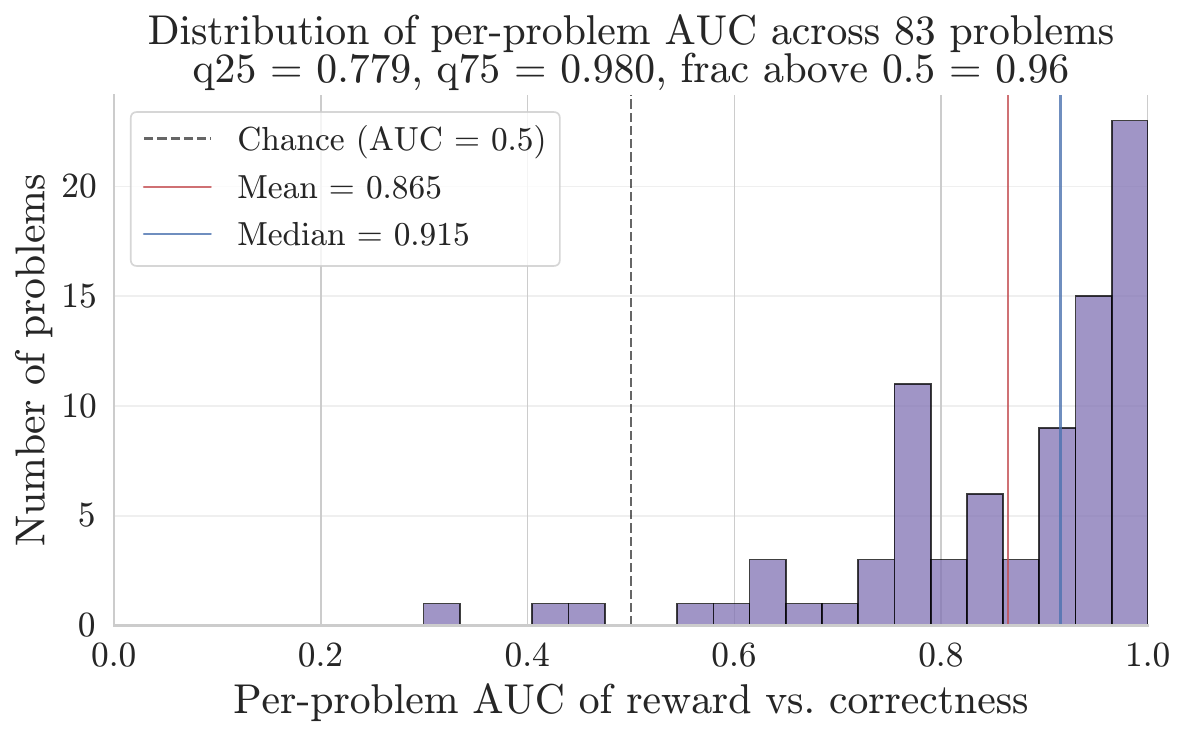}
        \caption{Coding (83 problems).}
    \end{subfigure}
    \hfill
    \begin{subfigure}[b]{0.495\linewidth}
        \includegraphics[width=\linewidth]{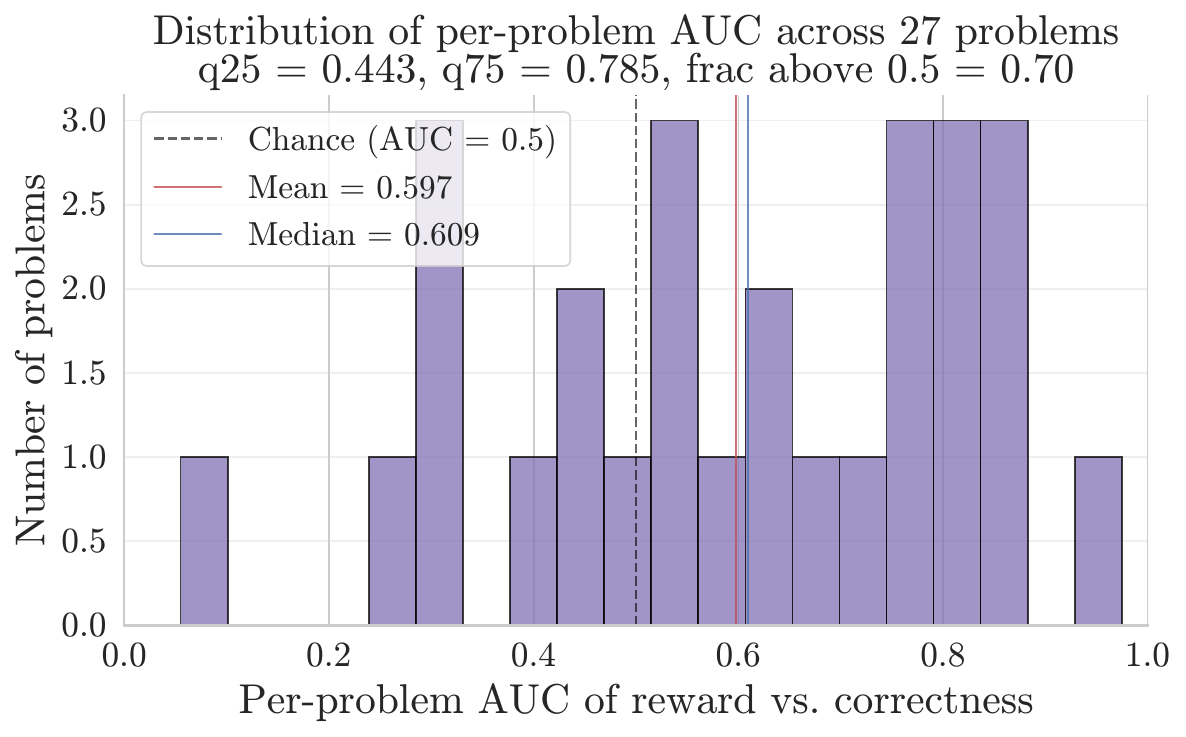}
        \caption{Math (27 problems).}
    \end{subfigure}
    \caption{Per-problem AUC of reward vs.\ correctness.
    Each bar counts the number of problems whose reward model achieves that AUC when ranking samples by score.
    The dashed vertical line marks chance ($\mathrm{AUC}=0.5$); red and blue verticals mark the mean and median.
    Coding rewards are strongly discriminative (mean $0.865$, $96\%$ above chance); math rewards are weaker and more variable (mean $0.597$, $70\%$ above chance).}
    \label{fig:auc-hist}
\end{figure}

\subsection{Robust difficulty-class benchmark}
\label{app:dap-robust}
The complementary benchmark $\DAP_k^{\mathrm{robust}}$ jointly chooses one contiguous partition $G_1,\ldots,G_k$ of the same fixed empirical-difficulty order and one pair $(m_i,v_i)=((N_{\mathrm{rew}})_i,(N_{\mathrm{ver}})_i)$ per class, with integers $1\le v_i\le m_i\le512$. It minimizes
\[
 \frac{1}{P}\sum_{i=1}^k|G_i|(c_{\mathrm{rew}}m_i+c_{\mathrm{ver}}v_i),
 \qquad S_{p,t}(m_i,v_i)=1\quad\text{for all }i\in[k],\ p\in G_i,\ t\in[T].
\]
Here $P$ is the number of problems, $T=10$ is the number of sampled orderings, and $S_{p,t}(m,v)$ indicates whether verifying the $v$ highest-reward responses among the first $m$ draws succeeds on problem $p$ in ordering $t$.
This is the original all-orderings DAP constraint. Relative to the per-ordering oracle, its choices have less freedom, so its optimal cost is at least as large. Both benchmarks use hidden verifier outcomes and are hindsight comparators. Its cost depends on $T$. Robust DAP protects only against the $T$ observed orderings, with no guarantee on unseen ones. At $k=P$, it still needs one budget per prompt that succeeds on every ordering; it need not equal \SAP. Counts, as in Table~\ref{tab:summary}, are weighted by class size and averaged over all prompt--ordering trials.

\subsection{Ablations}
\label{app:ablation}

\subsubsection{Alternative math generators}
\label{app:other-math-models}

The main body reports math results with \texttt{Qwen2.5-Math-7B} as the generator. We additionally evaluate \texttt{Qwen2.5-14B} (a stronger, non-math-specialized base model) and \texttt{DeepSeek-R1-Distill-Qwen-7B} model. Both use identical sampling settings ($N=512$, temperature $0.7$, top-$p$ $0.95$), the same reward model (\texttt{Qwen2.5-Math-PRM-7B}), and exact answer-string match as the verifier.

Table~\ref{tab:ablation-models} reports the number of solvable problems and mean cost of each policy per generator. The qualitative picture is unchanged across generators: rank-order monotonicity of the reward signal holds in all cases, and \Adap achieves $100\%$ success at substantially 
{lower cost than $\DAP_1$, demonstrating that the approach is not specific to the primary generator and generalizes reliably across models of varying size and specialization.}

\begin{table}[ht]
    \small
    \caption{$c_{\mathrm{ver}}/c_{\mathrm{rew}} = 10$. Per-generator results on HMMT. Mean cost averaged over solvable problems and $10$ permutations each. The ratio column shows cost relative to \Adap for the same generator.}
    \label{tab:ablation-models}
    \centering

\begin{tabular}{lrrrrrr}
\toprule
& \multicolumn{2}{c}{Qwen2.5-Math-7B} & \multicolumn{2}{c}{Qwen2.5-14B} & \multicolumn{2}{c}{DeepSeek-R1-7B} \\
\cmidrule(lr){2-3}\cmidrule(lr){4-5}\cmidrule(lr){6-7}
Policy & Cost & Ratio & Cost & Ratio & Cost & Ratio \\
\midrule
\# Solvable & 27 && 22 && 22 & \\
$\SAP$ & 601.7 & $\times0.33$ & 487.5 & $\times0.32$ & 476.6 & $\times0.33$ \\
$\Adap$ & 1819.1 & $\times1.00$ & 1508.5 & $\times1.00$ & 1443.6 & $\times1.00$ \\
\midrule
$\DAP_{1}$ & 4596.8 & $\times2.53$ & 4334.2 & $\times2.87$ & 4058.8 & $\times2.81$ \\
$\DAP_{2}$ & 2802.4 & $\times1.54$ & 2029.6 & $\times1.35$ & 1556.9 & $\times1.08$ \\
$\DAP_{3}$ & 2175.5 & $\times1.20$ & 1430.2 & $\times0.95$ & 1150.5 & $\times0.80$ \\
$\DAP_{5}$ & 1583.6 & $\times0.87$ & 1015.8 & $\times0.67$ & 840.4 & $\times0.58$ \\
$\DAP_{10}$ & 933.5 & $\times0.51$ & 641.6 & $\times0.43$ & 566.3 & $\times0.39$ \\
\midrule
$\DAP_{1}^{\mathrm{robust}}$ & 5330.0 & $\times2.93$ & 4661.0 & $\times3.09$ & 5532.0 & $\times3.83$ \\
$\DAP_{2}^{\mathrm{robust}}$ & 3730.0 & $\times2.05$ & 2192.5 & $\times1.45$ & 1980.3 & $\times1.37$ \\
$\DAP_{3}^{\mathrm{robust}}$ & 3014.6 & $\times1.66$ & 1716.7 & $\times1.14$ & 1691.6 & $\times1.17$ \\
$\DAP_{5}^{\mathrm{robust}}$ & 2403.1 & $\times1.32$ & 1350.8 & $\times0.90$ & 1393.6 & $\times0.97$ \\
$\DAP_{10}^{\mathrm{robust}}$ & 1690.7 & $\times0.93$ & 1134.0 & $\times0.75$ & 1149.6 & $\times0.80$ \\
\midrule
$\Uni_{C_{\Adap}}$ success & 0.744 &  & 0.864 &  & 0.827 &  \\
\bottomrule
\end{tabular}
\end{table}

\subsubsection{Cost-ratio robustness}
\label{app:cost-sweep}

The main body fixes $c_{\mathrm{ver}}/c_{\mathrm{rew}} = 10$. We repeat the comparison for $c_{\mathrm{ver}}/c_{\mathrm{rew}} \in \{1, 10, 20, 30\}$ on both tasks to verify that \Adap's advantage is not specific to this calibration. Tables~\ref{tab:cost-sweep} and~\ref{tab:cost-sweep-code} summarize the results.

\Adap consistently achieves lower cost than $\DAP_1$ at every ratio on both tasks, and the advantage on coding grows as verification becomes more expensive: $\DAP_1$ costs $3.3\times$ as much as \Adap at $c_{\mathrm{ver}}/c_{\mathrm{rew}}=1$, rising to $6.7\times$ at ratio $30$.
The $\Uni_{C_{\Adap}}$ rows confirm that a fixed policy given the same total budget as \Adap achieves only $73.7$--$90.4\%$ success, underscoring that the savings come from genuine adaptivity rather than simply spending less.
Notably, \Adap has lower cost than $\DAP_{10}^{\mathrm{robust}}$ on coding at all ratios, without requiring any offline difficulty estimation or problem partitioning.
The stronger $\DAP_{10}$ is cheaper than \Adap on both tasks at all four ratios.

\begin{table}[ht]
    \small
    \caption{Cost-ratio ablation on math (HMMT, $27$ problems, $10$ permutations each). The ratio column shows cost relative to \Adap at the same configuration.}
    \label{tab:cost-sweep}
    \centering
\begin{tabular}{lrrrrrrrr}
\toprule
& \multicolumn{2}{c}{$c_{\rm ver}=1$} & \multicolumn{2}{c}{$c_{\rm ver}=10$} & \multicolumn{2}{c}{$c_{\rm ver}=20$} & \multicolumn{2}{c}{$c_{\rm ver}=30$} \\
\cmidrule(lr){2-3}\cmidrule(lr){4-5}\cmidrule(lr){6-7}\cmidrule(lr){8-9}
Policy & Cost & Ratio & Cost & Ratio & Cost & Ratio & Cost & Ratio \\
\midrule
$\SAP$ & 176.6 & $\times0.51$ & 601.7 & $\times0.33$ & 1058.6 & $\times0.33$ & 1511.8 & $\times0.32$ \\
$\Adap$ & 349.5 & $\times1.00$ & 1819.1 & $\times1.00$ & 3203.1 & $\times1.00$ & 4794.4 & $\times1.00$ \\
$\Uni_{C_{\Adap}}$ success & 0.737 &  & 0.744 &  & 0.741 &  & 0.741 &  \\
\midrule
$\DAP_{1}$ & 854.6 & $\times2.45$ & 4596.8 & $\times2.53$ & 8754.8 & $\times2.73$ & 12912.8 & $\times2.69$ \\
$\DAP_{2}$ & 553.0 & $\times1.58$ & 2802.4 & $\times1.54$ & 5272.3 & $\times1.65$ & 7738.4 & $\times1.61$ \\
$\DAP_{3}$ & 449.4 & $\times1.29$ & 2175.5 & $\times1.20$ & 4073.3 & $\times1.27$ & 5967.1 & $\times1.24$ \\
$\DAP_{5}$ & 359.0 & $\times1.03$ & 1583.6 & $\times0.87$ & 2902.8 & $\times0.91$ & 4215.0 & $\times0.88$ \\
$\DAP_{10}$ & 251.3 & $\times0.72$ & 933.5 & $\times0.51$ & 1649.2 & $\times0.51$ & 2356.2 & $\times0.49$ \\
\midrule
$\DAP_{1}^{\mathrm{robust}}$ & 992.0 & $\times2.84$ & 5330.0 & $\times2.93$ & 10150.0 & $\times3.17$ & 14970.0 & $\times3.12$ \\
$\DAP_{2}^{\mathrm{robust}}$ & 740.9 & $\times2.12$ & 3730.0 & $\times2.05$ & 6950.0 & $\times2.17$ & 10170.0 & $\times2.12$ \\
$\DAP_{3}^{\mathrm{robust}}$ & 651.9 & $\times1.87$ & 3014.6 & $\times1.66$ & 5521.6 & $\times1.72$ & 8028.7 & $\times1.67$ \\
$\DAP_{5}^{\mathrm{robust}}$ & 567.6 & $\times1.62$ & 2403.1 & $\times1.32$ & 4299.8 & $\times1.34$ & 6196.4 & $\times1.29$ \\
$\DAP_{10}^{\mathrm{robust}}$ & 476.6 & $\times1.36$ & 1690.7 & $\times0.93$ & 2913.9 & $\times0.91$ & 4129.9 & $\times0.86$ \\
\bottomrule
\end{tabular}
\end{table}

\begin{table}[ht]
    \small
    \caption{Cost-ratio ablation on coding (LiveCodeBench, $83$ problems, $10$ permutations each). The ratio column shows cost relative to \Adap at the same configuration.}
    \label{tab:cost-sweep-code}
    \centering
    \begin{tabular}{lrrrrrrrr}
\toprule
& \multicolumn{2}{c}{$c_{\rm ver}=1$} & \multicolumn{2}{c}{$c_{\rm ver}=10$} & \multicolumn{2}{c}{$c_{\rm ver}=20$} & \multicolumn{2}{c}{$c_{\rm ver}=30$} \\
\cmidrule(lr){2-3}\cmidrule(lr){4-5}\cmidrule(lr){6-7}\cmidrule(lr){8-9}
Policy & Cost & Ratio & Cost & Ratio & Cost & Ratio & Cost & Ratio \\
\midrule
$\SAP$ & 97.6 & $\times0.40$ & 179.6 & $\times0.24$ & 265.2 & $\times0.25$ & 348.8 & $\times0.22$ \\
$\Adap$ & 241.4 & $\times1.00$ & 747.6 & $\times1.00$ & 1074.4 & $\times1.00$ & 1564.1 & $\times1.00$ \\
$\Uni_{C_{\Adap}}$ success & 0.795 &  & 0.880 &  & 0.880 &  & 0.904 &  \\
\midrule
$\DAP_{1}$ & 804.3 & $\times3.33$ & 3787.8 & $\times5.07$ & 7102.8 & $\times6.61$ & 10417.8 & $\times6.66$ \\
$\DAP_{2}$ & 409.1 & $\times1.69$ & 1692.6 & $\times2.26$ & 3072.6 & $\times2.86$ & 4452.2 & $\times2.85$ \\
$\DAP_{3}$ & 319.4 & $\times1.32$ & 1183.9 & $\times1.58$ & 2100.2 & $\times1.95$ & 3012.3 & $\times1.93$ \\
$\DAP_{5}$ & 248.8 & $\times1.03$ & 799.0 & $\times1.07$ & 1378.1 & $\times1.28$ & 1947.0 & $\times1.24$ \\
$\DAP_{10}$ & 186.1 & $\times0.77$ & 473.3 & $\times0.63$ & 769.0 & $\times0.72$ & 1059.5 & $\times0.68$ \\
\midrule
$\DAP_{1}^{\mathrm{robust}}$ & 871.0 & $\times3.61$ & 4102.0 & $\times5.49$ & 7692.0 & $\times7.16$ & 11282.0 & $\times7.21$ \\
$\DAP_{2}^{\mathrm{robust}}$ & 498.3 & $\times2.06$ & 2162.1 & $\times2.89$ & 3889.7 & $\times3.62$ & 5617.3 & $\times3.59$ \\
$\DAP_{3}^{\mathrm{robust}}$ & 413.4 & $\times1.71$ & 1501.5 & $\times2.01$ & 2604.7 & $\times2.42$ & 3708.0 & $\times2.37$ \\
$\DAP_{5}^{\mathrm{robust}}$ & 335.3 & $\times1.39$ & 1183.3 & $\times1.58$ & 1984.7 & $\times1.85$ & 2784.8 & $\times1.78$ \\
$\DAP_{10}^{\mathrm{robust}}$ & 281.9 & $\times1.17$ & 764.8 & $\times1.02$ & 1222.2 & $\times1.14$ & 1661.4 & $\times1.06$ \\
\bottomrule
\end{tabular}

\end{table}

\section{Appendix of \cref{sec:distribution-aware}} \label{sec:proof-distribution-aware}
\begin{lemma} \label{lem:optimal-streaming-oracle}
Fix an active search instance \((h^\star,\Dist)\) from \cref{def:active-search-instance} with \(R\sim \Dist\). Consider the class \(\Pi_{\rm str}(\Dist,h^\star)\) of streaming policies defined as follows. At each round \(i\), the policy draws a fresh score \(R_i\sim \Dist\), pays
\(\rewcost\), observes \(R_i\), and then chooses an action $ A_i\in\{\textsf{discard},\textsf{verify}\}$ as a possibly randomized function of the past history and the current score
\(R_i\). If \(A_i=\textsf{discard}\), the candidate is discarded permanently. If
\(A_i=\textsf{verify}\), the policy pays \(\vercost\) and observes $
    V_i\mid R_i \sim \mathrm{Bernoulli}(h^\star(R_i)).
$
The policy stops only when it verifies a candidate with \(V_i=1\). Then there exists a unique \(\tau^\star\in(0,1)\) satisfying
\[
    \vercost\,\mathbb E_{R\sim D}\!\left[(h^\star(R)-\tau^\star)_+\right]
    =
    \tau^\star \rewcost .
\]
Moreover,
\[
    J_{\rm str}^\star(h,D)=\frac{\vercost}{\tau^\star}.
\]
An optimal streaming policy is the threshold rule
\[
    A^\star(r)
    =
    \begin{cases}
        \textsf{verify}, & h(r)\geq\tau^\star,\\
        \textsf{discard}, & h(r)<\tau^\star.
    \end{cases}
\]

\end{lemma}
\begin{proof}
Let $(R_i,V_i)_{i\ge 1}$ be i.i.d.\ copies of $(R,V)$.
A (possibly randomized, history-dependent) policy $\pi$ chooses at each round $i$ an action
$A_i\in\{\textsf{discard},\textsf{verify}\}$ after observing $R_i$ (and the past history).
Let
\[
T:=\inf\{i\ge 1:\ A_i=\textsf{verify}\ \text{and}\ V_i=1\}
\]
be the stopping round, and define the total cost
\[
\mathrm{Cost}(\pi):=\sum_{i=1}^{T}\Bigl(\rewcost+\vercost\mathbf{1}\{A_i=\textsf{verify}\}\Bigr),
\qquad
J(\pi):=\mathbb{E}[\mathrm{Cost}(\pi)],
\qquad
J^*:=\inf_{\pi} J(\pi).
\]

Fix a round, and condition on having already paid $\rewcost$ and observed $R=r$.

\begin{itemize}
\item If we \textsf{discard}, we move to the start of the next round, whose optimal expected remaining
cost is $J^*$.

\item If we \textsf{verify}, we pay $\vercost$ and observe $V$. With probability $h(r)$ we stop
immediately; with probability $1-h(r)$ we fail and return to a fresh round with optimal expected remaining
cost $J^*$. Thus the continuation cost is
\[
\vercost+(1-h(r))J^*.
\]
\end{itemize}

Therefore, verifying is optimal iff
\[
\vercost+(1-h(r))J^*\le J^*
\quad\Longleftrightarrow\quad
h(r)\ge \frac{\vercost}{J^*}.
\]
Define
\begin{equation}\label{eq:tau-def}
\tau^*:=\frac{\vercost}{J^*}\in(0,1].
\end{equation}
This shows there exists an optimal policy that verifies exactly when $h(r)\ge \tau^*$.

To characterize $\tau^*$, we use Bellman equation. At the start of a fresh round, we pay $\rewcost$, observe $R$, and then choose the smaller of
discarding (continuation cost $J^*$) and verifying (continuation cost $\vercost+(1-h(R))J^*$).
Thus $J^*$ satisfies the Bellman identity
\begin{equation}\label{eq:bellman}
J^*
=
\rewcost
+
\mathbb{E}\Bigl[\min\{J^*,\ \vercost+(1-h(R))J^*\}\Bigr].
\end{equation}

For any $J>0$ and $x\in[0,1]$,
\[
\min\{J,\ \vercost+(1-x)J\}
=
J-(xJ-\vercost)_+,
\]
where $ (u)_+:=\max\{u,0\}$. Apply this with $x=h(R)$ and $J=J^*$ in \eqref{eq:bellman}:
\[
J^*
=
\rewcost+\mathbb{E}\bigl[J^*-(h(R)J^*-\vercost)_+\bigr]
\quad\Longleftrightarrow\quad
\rewcost=\mathbb{E}\bigl[(h(R)J^*-\vercost)_+\bigr].
\]
Using $\tau^*=\vercost/J^*$ from \eqref{eq:tau-def}, we rewrite
\[
(h(R)J^*-\vercost)_+ = J^*(h(R)-\tau^*)_+,
\]
and hence
\[
\rewcost
=
J^*\,\mathbb{E}\bigl[(h(R)-\tau^*)_+\bigr]
=
\frac{\vercost}{\tau^*}\,\mathbb{E}\bigl[(h(R)-\tau^*)_+\bigr].
\]
Multiplying both sides by $\tau^*$ gives
\begin{equation}\label{eq:fp-expectation}
\tau^* \rewcost = \vercost\,\mathbb{E}\bigl[(h(R)-\tau^*)_+\bigr].
\end{equation}
Also, \eqref{eq:tau-def} yields
\[
J^*=\frac{\vercost}{\tau^*}.
\]

Then, we show that $\tau^\star$ is unique and exists. Define for $\tau\in[0,1]$,
\[
g(\tau):=
\vercost \int_{\{r:\,h(r)\ge \tau\}} (h(r)-\tau) \text{d}\Dist(r)
-\tau \rewcost
=
\vercost\mathbb{E}\bigl[(h(R)-\tau)_+\bigr]-\tau \rewcost.
\]
Then $g$ is continuous. Moreover, for $0\le \tau_1<\tau_2\le 1$,
\[
g(\tau_2)-g(\tau_1)
&=
\vercost\Bigl(\mathbb{E}[(h(R)-\tau_2)_+]-\mathbb{E}[(h(R)-\tau_1)_+]\Bigr)
-(\tau_2-\tau_1)\rewcost \\
&\leq -(\tau_2-\tau_1)\rewcost<0,
\]
so $g$ is strictly decreasing. Also,
\[
g(0)=\vercost\EE[h(R)]=\vercost\Pr(V=1)>0,
\qquad
g(1)=0-\rewcost<0.
\]
By the intermediate value theorem there exists a root $\tau^*\in(0,1]$, and by strict monotonicity
it is unique.

\end{proof}

\begin{lemma} \label{lem:distribution-aware-online-optimal}
Fix an active search instance \((h,\Dist)\) from
\cref{def:active-search-instance}. Among all distribution-aware pool-based
policies in the sense of \cref{def:active-search}, there exists an optimal
policy that never stores candidates for later use. In particular, the
streaming threshold policy of \cref{lem:optimal-streaming-oracle} is optimal
within this broader class.
\end{lemma}

\begin{proof}
Let \(J^\star=\vercost/\tau^\star\) be the optimal streaming value from
\cref{lem:optimal-streaming-oracle}. Verifying a sequence of candidates is
equivalent to verifying them one at a time, so it suffices to compare the two
elementary actions: (i)~generating a finite batch, and (ii)~verifying a single
stored candidate.

\paragraph{The candidate continuation value.}
Given a finite pool \(P\), let \(r_{(1)},\ldots,r_{(k)}\) enumerate its
above-threshold elements in decreasing order of \(h\):
\[
    h(r_{(1)})\ge\cdots\ge h(r_{(k)})>\tau^\star,
\]
with \(k=0\) if no such elements exist. Define
\begin{equation}\label{eq:W-def}
    W(P) :=
    \vercost \sum_{i=1}^{k} \prod_{\ell<i}\bigl(1-h(r_{(\ell)})\bigr)
    \;+\;
    J^\star \prod_{i=1}^{k}\bigl(1-h(r_{(i)})\bigr).
\end{equation}
This is the expected cost of the policy that verifies the above-threshold
candidates in decreasing order of \(h\), discards the rest, and falls back to a
fresh start (value \(J^\star\)) if all verifications fail. We show that
\(W(P)\) is the optimal continuation value from \(P\) by establishing two
structural properties of \(W\) and then comparing the elementary actions.

\paragraph{Step 1: \(W(P)\le J^\star\).}
Set \(p_i=h(r_{(i)})\) and define \(U_{k+1}=J^\star\),
\(U_i=\vercost+(1-p_i)U_{i+1}\), so that \(W(P)=U_1\). Since
\(p_i>\tau^\star=\vercost/J^\star\), backward induction on \(i\) gives
\[
    U_i \le \vercost + (1-p_i)J^\star = J^\star - (p_iJ^\star-\vercost) < J^\star,
\]
and hence \(W(P)\le J^\star\).

\paragraph{Step 2: Insertion bound.}
For every deterministic score \(r\),
\begin{equation}\label{eq:insertion}
    W(P) - W(P\cup\{r\}) \le \bigl(h(r)J^\star - \vercost\bigr)_+.
\end{equation}
If \(h(r)\le\tau^\star\), then \(r\) does not appear in
\eqref{eq:W-def}, so \(W(P\cup\{r\})=W(P)\) and the bound is immediate.
Suppose instead that \(p:=h(r)>\tau^\star\), and insert \(r\) into the ordered
above-threshold list. Let \(S\le1\) denote the probability that all candidates
preceding \(r\) fail, and let \(U\) denote the tail value at the insertion
point before adding \(r\). A direct calculation from \eqref{eq:W-def} yields
\[
    W(P) - W(P\cup\{r\}) = S\bigl(pU-\vercost\bigr).
\]
The candidates following \(r\) all have success probability at most \(p\), and
the fall-back value satisfies \(J^\star\ge \vercost/p\); a backward induction
along the tail then gives \(\vercost/p\le U\le J^\star\). Combined with
\(S\le1\), this yields
\(0\le S(pU-\vercost)\le pJ^\star-\vercost\), proving \eqref{eq:insertion}.

\paragraph{Step 3: No first action improves on \(W(P)\).}
We bound the expected cost of each elementary action from \(P\).

\emph{Generate.} Suppose the policy generates a batch
\(R_1,\ldots,R_m\sim\Dist\). Iterating \eqref{eq:insertion},
\[
    W(P\cup\{R_1,\ldots,R_m\}) \ge W(P) - \sum_{i=1}^{m}\bigl(h(R_i)J^\star-\vercost\bigr)_+.
\]
The fixed-point identity from \cref{lem:optimal-streaming-oracle},
\(\tau^\star\rewcost=\vercost\,\EE[(h(R)-\tau^\star)_+]\), is equivalent to
\(\rewcost=\EE\bigl[(h(R)J^\star-\vercost)_+\bigr]\) after multiplying by
\(J^\star\). Therefore, taking expectation,
\[
    m\rewcost + \EE\bigl[W(P\cup\{R_1,\ldots,R_m\})\bigr]
    \;\ge\;
    W(P) + m\rewcost - m\,\EE\bigl[(h(R)J^\star-\vercost)_+\bigr]
    \;=\;
    W(P).
\]
Hence generating first cannot improve on \(W(P)\).

\emph{Verify a below-threshold candidate.} If \(h(r)\le\tau^\star\), then
\(W(P\setminus\{r\})=W(P)\), and the expected cost of verifying \(r\) is
\[
    \vercost + (1-h(r))W(P)
    \;=\; W(P) + \vercost - h(r)W(P)
    \;\ge\; W(P),
\]
where the last inequality uses \(W(P)\le J^\star\) and
\(h(r)\le\tau^\star=\vercost/J^\star\), so that \(h(r)W(P)\le\vercost\).

\emph{Verify an above-threshold candidate.} Among such candidates, verifying
in decreasing order of \(h\) is optimal. Indeed, for \(p\ge q\), verifying
\(p\) before \(q\) costs \(\vercost+(1-p)\bigl(\vercost+(1-q)U\bigr)\),
whereas the reverse order costs \(\vercost+(1-q)\bigl(\vercost+(1-p)U\bigr)\);
the difference is \(\vercost(q-p)\le 0\). The decreasing-\(h\) order attains
exactly \(W(P)\) by construction.

\paragraph{Step 4: Optimality of \(W(P)\).}
Steps~1–3 show that every elementary action has expected continuation cost at
least \(W(P)\). To extend this one-step comparison to arbitrary policies, fix
any pool-based policy \(\pi\) starting from \(P\). Let \(C_t\) denote the cost
accumulated after \(t\) decisions, \(P_t\) the pool before the \(t\)-th
decision, and \(T\) the first successful verification time. By induction on
\(t\),
\[
    W(P) \;\le\; \EE_\pi\!\Bigl[ C_t + W(P_t)\,\mathbf 1\{T>t\} \Bigr]
    \qquad \forall t\ge1.
\]
If \(\EE_\pi[C_T]=\infty\), there is nothing to prove. Otherwise \(T<\infty\)
almost surely, since each action costs at least
\(\min\{\rewcost,\vercost\}>0\). Letting \(t\to\infty\) and using
\(0\le W(P_t)\le J^\star\) gives \(W(P)\le \EE_\pi[C_T]\). Thus no pool-based
policy beats \(W(P)\), and the policy defining \(W(P)\) attains it.

\paragraph{Conclusion.}
For \(P=\emptyset\), \eqref{eq:W-def} reduces to \(W(\emptyset)=J^\star\).
Hence an optimal pool-based policy can be implemented by generating a single
fresh score \(r\), verifying it when \(h(r)>\tau^\star\), discarding it when
\(h(r)<\tau^\star\), and breaking ties arbitrarily at \(h(r)=\tau^\star\).
This is exactly the streaming policy of \cref{lem:optimal-streaming-oracle}.
\end{proof}

\section{Appendix of \cref{sec:non-dec-alg}} \label{sec:proof-adap}

\begin{proof}[Proof of \cref{thm:expected-shellwise}]
Fix a feasible prompt \(x\), and suppress the dependence on \(x\). Let
\((h^\star,\Dist)\) be the induced active-search instance. Soundness is
immediate from the definition of \(\Alg\): the algorithm returns only after
querying the verifier and observing label \(1\).

We now prove the expected-cost bound.

\paragraph{Step 1: Reduce the oracle to reward thresholds.}
Since \(h^\star\) is non-decreasing, \cref{lem:monotone-threshold-reduction}
implies that the distribution-aware optimum is exactly the infimum over
reward-threshold policies:
\[
    J^\star(h^\star,\Dist)
    =
    \inf_{t:\,s_t>0} J_t .
\]

\paragraph{Step 2: Compare ADAP to any fixed reward threshold.}
By \cref{lem:expectation-fixed-threshold}, for every threshold \(t\) with
\(s_t>0\),
\[
    \mathbb E[J(\Alg;h^\star,\Dist)]
    \le
    154 J_t .
\]

\paragraph{Step 3: Optimize over the threshold.}
Since the previous inequality holds for every threshold \(t\) with \(s_t>0\),
we may take the infimum over such thresholds:
\[
    \mathbb E[J(\Alg;h^\star,\Dist)]
    \le
    154\inf_{t:\,s_t>0}J_t .
\]
Using \cref{lem:monotone-threshold-reduction} again,
\[
    \mathbb E[J(\Alg;h^\star,\Dist)]
    \le
    154 J^\star(h^\star,\Dist).
\]
This proves the theorem.
\end{proof}

\begin{lemma}[Monotone oracle reduces to reward thresholds]
\label{lem:monotone-threshold-reduction}
Fix an active search instance \((h^\star,\Dist)\) with \(R\sim\Dist\) and
\(\Pr(V=1)>0\). Assume that \(h^\star\) is non-decreasing. For each threshold
\(t\in\mathbb R\), define
\[
    q_t := \Pr_{R\sim\Dist}(R\ge t),
    \qquad
    s_t :=
    \mathbb E_{R\sim\Dist}
    \left[
        h^\star(R)\mathbf 1\{R\ge t\}
    \right],
\]
and, whenever \(s_t>0\),
\[
    J_t := \frac{\rewcost+\vercost q_t}{s_t}.
\]
Then the distribution-aware optimum is the infimum over reward-threshold
policies:
\[
    J^\star(h^\star,\Dist)
    =
    \inf_{t:\,s_t>0} J_t .
\]
\end{lemma}

\begin{proof}
By \cref{lem:distribution-aware-online-optimal,lem:optimal-streaming-oracle}, an optimal
distribution-aware policy is a (streaming) threshold rule in the success probability: there is a
break-even $\tau^\star$ such that the policy verifies a candidate of score $r$ iff $h^\star(r)>\tau^\star$,
with arbitrary tie-breaking when $h^\star(r)=\tau^\star$. Since $h^\star$ is non-decreasing, the set
$U^\star := \{r\in\mathbb{R}: h^\star(r)>\tau^\star\}$ is an \emph{upper set}.

It is convenient to analyze policies of the following form. For any measurable $U\subseteq\mathbb{R}$,
consider the policy that repeatedly generates $R\sim \Dist$, pays $\rewcost$, verifies iff $R\in U$
(paying $\vercost$), and stops upon the first positive verifier label. Let
$q(U):=\Pr(R\in U)$ and $s(U):=\EE[h^\star(R)\mathbf{1}\{R\in U\}]$. Writing $J(U)$ for the expected
cost of this policy, we have the one-step recursion
\[
J(U)=\rewcost+\vercost\,q(U) + (1-s(U))\,J(U),
\]
since each round incurs expected cost $\rewcost+\vercost q(U)$ and succeeds with probability $s(U)$.
Solving gives
\[
J(U)=\frac{\rewcost+\vercost\,q(U)}{s(U)} \qquad\text{whenever } s(U)>0.
\]

Now specialize to $U^\star$. Define $q_\star:=q(U^\star)=\Pr(R\in U^\star)$ and
$s_\star:=s(U^\star)=\EE[h^\star(R)\mathbf{1}\{R\in U^\star\}]$. By \cref{lem:optimal-streaming-oracle},
$\tau^\star$ satisfies the break-even identity
$\vercost\,\EE[(h^\star(R)-\tau^\star)_+]=\tau^\star \rewcost$ and the optimal value is
$J^\star(h^\star,\Dist)=\vercost/\tau^\star$. Moreover, since $U^\star=\{r:h^\star(r)>\tau^\star\}$,
we can rewrite
\[
(h^\star(R)-\tau^\star)_+ = (h^\star(R)-\tau^\star)\mathbf{1}\{R\in U^\star\},
\]
so $\EE[(h^\star(R)-\tau^\star)_+] = s_\star-\tau^\star q_\star$. Plugging into the break-even
identity yields $\vercost(s_\star-\tau^\star q_\star)=\tau^\star\rewcost$, i.e.,
\[
\vercost s_\star=\tau^\star(\rewcost+\vercost q_\star).
\]
Dividing by $\tau^\star s_\star$ gives
\[
\frac{\rewcost+\vercost q_\star}{s_\star}=\frac{\vercost}{\tau^\star}=J^\star(h^\star,\Dist).
\]
In other words, the upper-set policy associated with $U^\star$ attains the distribution-aware optimum.

It remains to relate this upper-set policy to \emph{ordinary reward thresholds}. Since $U^\star$ is an
upper subset of $\mathbb{R}$, there exists a sequence of thresholds $(t_n)_{n\ge 1}$ such that
$\{r:r\ge t_n\}\uparrow U^\star$ (e.g., if $U^\star=(t_0,\infty)$ take $t_n=t_0+1/n$, and if
$U^\star=[t_0,\infty)$ take $t_n=t_0$). Let $q_n:=\Pr(R\ge t_n)$ and
$s_n:=\EE[h^\star(R)\mathbf{1}\{R\ge t_n\}]$. By monotone convergence, $q_n\to q_\star$ and
$s_n\to s_\star$. Also $s_\star>0$: otherwise $\EE[(h^\star(R)-\tau^\star)_+]=0$, contradicting
$\tau^\star\rewcost>0$. Hence $s_n>0$ for all large $n$, and for those $n$,
\[
J_{t_n}=\frac{\rewcost+\vercost q_n}{s_n}\longrightarrow \frac{\rewcost+\vercost q_\star}{s_\star}
=J^\star(h^\star,\Dist).
\]
Therefore $\inf_{t:\,s_t>0} J_t \le J^\star(h^\star,\Dist)$.

For the reverse inequality, fix any threshold $t$ with $s_t>0$. The policy that generates i.i.d. and
verifies iff $R\ge t$ is a valid distribution-aware policy with expected cost $J_t$, so by definition
of $J^\star$ we have $J^\star(h^\star,\Dist)\le J_t$ for every such $t$. Taking the infimum over $t$
gives $J^\star(h^\star,\Dist)\le \inf_{t:\,s_t>0}J_t$. Combining both directions completes the proof.
\end{proof}

\begin{restatable}{lemma}{expectationDyadicLemma}
\label{lem:expectation-dyadic}
Let $q_t=\Pr\left(R\geq t\right)$ and $s_t = \Pr(R \ge t, V=1)$. For every threshold $t \in \Reals$ with \(q_t>0\) and \(s_t>0\), there are integers \(a_t,b_t\ge0\) such that
\[
    2^{-a_t-1}<q_t\le 2^{-a_t},\qquad
    2^{-b_t}\le s_t<2^{-b_t+1}.
\]
Moreover \(a_t\le b_t\), and \(J_t\le B_{a_t,b_t}\le 4J_t\) where $B_{a_t,b_t}=\rewcost 2^{b_t}+\vercost 2^{b_t-a_t}$.
\end{restatable}

\begin{proof}
The dyadic indices exist by construction. Since \(s_t\le q_t\), we must have \(a_t\le b_t\). For the first inequality, \(s_t\ge 2^{-b_t}\) and \(q_t\le 2^{-a_t}\) imply
$
    J_t=\frac{\rewcost+\vercost q_t}{s_t}
    \le
    (\rewcost+\vercost 2^{-a_t})2^{b_t}
    =
    B_{a_t,b_t}.
$
For the reverse inequality, \(s_t<2^{-b_t+1}\) gives \(2^{b_t}<2/s_t\), while \(q_t>2^{-a_t-1}\) gives \(2^{-a_t}<2q_t\). Therefore
$
    B_{a_t,b_t}
    =
    (\rewcost+\vercost 2^{-a_t})2^{b_t}
    <
    (\rewcost+2\vercost q_t)\frac{2}{s_t}
    \le
    4J_t .
$
\end{proof}

\begin{restatable}{lemma}{topRankedMonotonicityLemma}
\label{lem:top-ranked-monotonicity}
Assume \(h^\star\) is non-decreasing. Let \(A\subseteq B\) be finite multisets of reward scores. Fix integers \(k,\ell\) such that $k\le |A|$, and $k\le \ell\le |B|$.
Let \(A_k\) be the multiset of the \(k\) largest elements of \(A\), and let \(B_\ell\) be the multiset of the \(\ell\) largest elements of \(B\). Then
\[
    1-\prod_{r\in B_\ell}\bigl(1-h^\star(r)\bigr)
    \ge
    1-\prod_{r\in A_k}\bigl(1-h^\star(r)\bigr).
\]
\end{restatable}

\begin{proof}
Write the elements of \(A\) and \(B\) in non-increasing order:
\[
    a_{(1)}\ge a_{(2)}\ge \cdots \ge a_{(|A|)},
    \qquad
    b_{(1)}\ge b_{(2)}\ge \cdots \ge b_{(|B|)}.
\]
Since \(A\subseteq B\), for every \(j\le |A|\), we have  $b_{(j)}\ge a_{(j)}.$
In particular, this holds for every \(j\le k\). Since \(h^\star\) is non-decreasing, $h^\star(b_{(j)})\ge h^\star(a_{(j)})$ for all $j \le k$.
Therefore,
\(1-h^\star(b_{(j)}) \le 1-h^\star(a_{(j)})\) for all $j \le k$.
Multiplying over \(j=1,\dots,k\), we get
\[
    \prod_{j=1}^{k}\bigl(1-h^\star(b_{(j)})\bigr)
    \le
    \prod_{j=1}^{k}\bigl(1-h^\star(a_{(j)})\bigr).
\]
Because \(k\le \ell\), the product over the top \(\ell\) elements of \(B\)
contains these first \(k\) factors and additional factors in \([0,1]\). Hence
\[
    \prod_{j=1}^{\ell}\bigl(1-h^\star(b_{(j)})\bigr)
    \le
    \prod_{j=1}^{k}\bigl(1-h^\star(b_{(j)})\bigr)
    \le
    \prod_{j=1}^{k}\bigl(1-h^\star(a_{(j)})\bigr).
\]
Taking complements concludes the claim.
\end{proof}

\begin{restatable}{lemma}{expectationLaterShellsLemma}
\label{lem:expectation-later-shells}
Fix a threshold \(t\) with dyadic pair \((a,b)\), and let \(s_\star\) be the shell such that \((a,b)\in S_{s_\star}\). Then for every integer \(u\ge0\), \((a,b+u)\in S_{s_\star+u}\). Consequently, the stage in shell \(s_\star+u\) has
\[
    m_{s_\star+u}\ge \lceil 2^{b+u+1}\rceil,\qquad
    k_{s_\star+u}\ge \lceil 6\cdot 2^{b-a+u}\rceil.
\]
\end{restatable}
\begin{proof}
We have
\[
    B_{a,b+u}
    =
    \rewcost 2^{b+u}+\vercost 2^{b+u-a}
    =
    2^uB_{a,b}.
\]
Since \((a,b)\in S_{s_\star}\), \(2^{s_\star}c_{\min}\le B_{a,b}<2^{s_\star+1}c_{\min}\). Multiplying by \(2^u\) gives \(2^{s_\star+u}c_{\min}\le B_{a,b+u}<2^{s_\star+u+1}c_{\min}\). Hence \((a,b+u)\in S_{s_\star+u}\). Since this pair is in the shell, the definitions of \(b_{s_\star+u}^\star\) and \(j_{s_\star+u}^\star\) imply \(b_{s_\star+u}^\star\ge b+u\) and \(j_{s_\star+u}^\star\ge b-a+u\). The bounds on \(m_{s_\star+u}\) and \(k_{s_\star+u}\) follow.
\end{proof}

\begin{restatable}{lemma}{expectationShellFailureLemma}
\label{lem:expectation-shell-failure}
Fix a threshold \(t\) with dyadic pair \((a,b)\), and let \(s_\star\) be the matching shell. For every \(u\ge0\), conditional on the algorithm reaching shell \(s_\star+u\), the probability that shell \(s_\star+u\) fails to output a verified positive is at most
\[
    \rho_u=e^{-2^{u+1}}+\text{\(\tfrac13e^{-\beta2^{u+1}}\)},
    \qquad\beta:=3 \ln 3-2>1.
\]
\end{restatable}
\begin{proof}
{
By \cref{lem:expectation-later-shells}, shell \(s_\star+u\) generates at least 
\(m_u:= 2^{b+u+1}\) fresh samples and has verification budget at least \(k_u:=6\cdot 2^{b-a+u}\).
Condition on the history at the start of shell \(s_\star+u\). The fresh samples
generated in this shell are independent of this history. We analyze any fixed
subset of \(m_u\) fresh samples generated in the shell.

Let 
$$N_t:=\sum_{i=1}^{m_u}\mathbf 1\{R_i\ge t\},\quad\text{and}\quad Z_t:=\sum_{i=1}^{m_u}\mathbf 1\{R_i\ge t,V_i=1\}.$$
If \(Z_t\ge1\) and \(N_t\le k_u\), then the top \(k_u\) rewards among these
\(m_u\) fresh samples contain a verified-positive candidate. Since these fresh
samples are part of the pool and since the actual verification budget in the
shell is at least \(k_u\), \cref{lem:top-ranked-monotonicity} implies that the
conditional success probability of the actual pool-verification step is at least
the conditional probability of success from these \(m_u\) fresh samples alone.

}

{
It remains to bound the probability that the event \(Z_t\ge1\) and \(N_t\le k_u\) fails. Since \(Z_t\sim\mathrm{Bin}(m_u,s_t)\) and \(s_t\ge2^{-b}\), we have \(m_us_t\ge 2^{b+u+1}2^{-b}=2^{u+1}\). Thus
}
\[
    \Pr(Z_t=0)\le e^{-m_us_t}\le e^{-2^{u+1}}.
\]
Next, \(N_t\sim\mathrm{Bin}(m_u,q_t)\) and \(q_t\le2^{-a}\). Hence
\[
    \mathbb E[N_t]\le 2^{b+u+1}2^{-a} =2^{b-a+u+1},
\]
Let \(\mu:=2^{b-a+u+1}\), so \(\mathbb E[N_t]\le\mu\) and \(k_u=3\mu\). Each indicator \(\mathbf 1\{R_i\ge t\}\) contributes the expected factor \((1-q_t)3^0+q_t3^1=1+2q_t\). Independence therefore gives
\[
    \mathbb E[3^{N_t}]=(1+2q_t)^{m_u}
    \le e^{2m_uq_t}\le e^{2\mu},
\]
where the inequalities use \(1+x\le e^x\) and \(m_uq_t\le\mu\). Thus
\[
    \Pr(N_t>k_u)
    &=\Pr(N_t\ge3\mu+1)
    =\Pr(3^{N_t}\ge3^{3\mu+1})\\
    &\le\frac{\mathbb E[3^{N_t}]}{3^{3\mu+1}}
    \le\frac{e^{2\mu}}{3^{3\mu+1}}\\
    &=\frac13e^{-\beta\mu}
    \le\frac13e^{-\beta2^{u+1}}.
\]
The first equality uses that \(N_t\) and \(k_u=3\mu\) are integers; the second uses that \(x\mapsto3^x\) is increasing. The first inequality is Markov's inequality applied to the nonnegative \(3^{N_t}\) (the exponential Chernoff step), and the second uses the moment bound above. The last equality uses \(\beta=3 \ln 3-2\), and the last inequality uses \(\mu\ge2^{u+1}\), since \(b\ge a\).
A union bound gives the claim.
\end{proof}

\begin{restatable}{lemma}{expectationOneShellCostLemma}
\label{lem:expectation-one-shell-cost}
For every nonempty shell \(S_s\), the deterministic cost \(\Lambda_s:=\rewcost m_s+\vercost k_s\) of shell \(s\) satisfies \(\Lambda_s\le 16\cdot 2^s c_{\min}\).
\end{restatable}
\begin{proof}

Substituting definitions of $m_s$ and $k_s$ from Algorithm~\ref{alg:expected-shellwise-search}, we have
\[
    \Lambda_s
    =
    \rewcost(2^{b_s^\star+1})+\vercost(6\cdot2^{j_s^\star})
    \le
    2\rewcost2^{b_s^\star}+6\vercost2^{j_s^\star}.
\]
Since \(b_s^\star\) is attained by some pair in \(S_s\), \(\rewcost2^{b_s^\star}<2^{s+1}c_{\min}\). Similarly, since \(j_s^\star\) is attained by some pair in \(S_s\), \(\vercost2^{j_s^\star}<2^{s+1}c_{\min}\). Therefore \(\Lambda_s<2 \cdot 2^{s+1}c_{\min}+6 \cdot2^{s+1}c_{\min}=16\cdot2^s c_{\min}\).
\end{proof}

\begin{lemma}
\label{lem:expectation-tail}
Let \(\beta:=3\ln 3-2\) and
\[
    \Gamma=1+\sum_{u\ge1}2^u\left(e^{-2^u}+\tfrac13e^{-\beta2^u}\right).
\]
Then \(\Gamma<1.405\).
\end{lemma}

\begin{proof}

For \(a\ge1\), define \(g(a):=\sum_{u\ge1}2^ue^{-a2^u}\). Then
\[
    \Gamma=1+g(1)+\frac13g(\beta).
\]
We first bound \(g(a)\). Keep the first term \(2e^{-2a}\) and bound all remaining terms. Write \(u=j+2\), where \(j\ge0\) is an integer. Since \(2^j \ge j+1\),
\[
    2^{j+2}e^{-a2^{j+2}}
    &=4\cdot2^j e^{-4a2^j}\\
    &\le4\cdot2^j e^{-4a(j+1)}
    =4e^{-4a}(2e^{-4a})^j.
\]
Thus each remaining term is bounded by the corresponding term of a geometric series with first term \(4e^{-4a}\) and ratio \(2e^{-4a}<1\). Summing these bounds gives
\[
    g(a)
    \le2e^{-2a}+4e^{-4a}\sum_{j\ge0}(2e^{-4a})^j
    =2e^{-2a}+\frac{4e^{-4a}}{1-2e^{-4a}}.
\]
Applying this bound at \(a=1\) and \(a=\beta>1\), we obtain
\[
    \Gamma
    &=1+g(1)+\frac13g(\beta)\\
    &\le1+2e^{-2}+\frac{4e^{-4}}{1-2e^{-4}}
    +\frac13\left(2e^{-2\beta}+\frac{4e^{-4\beta}}{1-2e^{-4\beta}}\right)\\
    &<1.40425<1.405.
\]
The numerical bound follows by substituting \(e^{-2}<0.13534\) and \(e^{-2\beta}=e^4/729<0.07492\); each summand is increasing in the corresponding exponential term.

\end{proof}

\begin{restatable}{lemma}{expectationFixedThresholdLemma}
\label{lem:expectation-fixed-threshold}
For every threshold \(t\) with \(s_t>0\), the expected cost of the shellwise algorithm is at most $154J_t$.
\end{restatable}
\begin{proof}
Let \((a,b)\) be the dyadic pair for \(t\), and let \(s_\star\) be the shell with \((a,b)\in S_{s_\star}\). By Lemma~\ref{lem:expectation-dyadic}, \(2^{s_\star}c_{\min}\le B_{a,b}\le4J_t\).

First consider the deterministic cost before shell \(s_\star\). By Lemma~\ref{lem:expectation-one-shell-cost},
\[
    \sum_{s=0}^{s_\star-1}\Lambda_s
    \le
    16c_{\min}\sum_{s=0}^{s_\star-1}2^s
    <
    16\cdot2^{s_\star}c_{\min}.
\]
Now consider the cost from shell \(s_\star\) onward. The algorithm reaches shell \(s_\star\) with probability at most one. For \(u\ge1\), in order to reach shell \(s_\star+u\), the algorithm must have reached and failed shell \(s_\star+u-1\). By Lemma~\ref{lem:expectation-shell-failure}, this conditional failure probability is at most \(\rho_{u-1}\). Therefore
\[
    \mathbb E[\text{cost from shell }s_\star\text{ onward}]
    \le
    16\cdot2^{s_\star}c_{\min}
    \left(1+\sum_{u\ge1}2^u\rho_{u-1}\right).
\]
Since \(\rho_{u-1}=e^{-2^u}+\frac{1}{3}e^{-\beta 2^u}\), Lemma~\ref{lem:expectation-tail} gives \(\Gamma=1+\sum_{u\ge1}2^u\rho_{u-1}<1.405\), and hence
\[
    \mathbb E[\text{cost from shell }s_\star\text{ onward}]
    \le
    \text{\(16\Gamma\)}\cdot2^{s_\star}c_{\min}.
\]
Combining the two parts,
\[
    \mathbb E[J(\Alg;h,D)]
    &\le
    16(1+\Gamma)\cdot2^{s_\star}c_{\min}
    \le
    16(1+\Gamma)B_{a,b}\\
    &\le
    64(1+\Gamma)J_t
    <153.92J_t<154J_t.
\]
\end{proof}

\section{Appendix of \cref{sec:seperation}}
\newcommand{\Hist}{\mathcal{H}}
\begin{proof}
Fix an integer $m \le \starnum_0$. By \Cref{def:starnum}, choose
$x_1,\dots,x_m \in \dataspace$ and $h_1,\dots,h_m \in \conceptclass$ with
$h_i(x_j) = \indic{i=j}$. Let $\Dist$ be uniform on $\{x_1,\dots,x_m\}$.

\paragraph{Oracle upper bound.} For any target $h_i$, the distribution-aware
oracle generates i.i.d.\ samples from $\Dist$ until $x_i$ appears, then
verifies it once. Since $\Dist(x_i) = 1/m$,
\begin{equation}\label{eq:oracle}
  J^\star(h_i,\Dist) \;\le\; \rewcost\, m + \vercost
  \qquad \text{for every } i \in [m].
\end{equation}

\paragraph{Algorithm lower bound.} Place a uniform prior $I \sim
\mathrm{Unif}([m])$ on the target index, and let $T$ be the first verifier
call that returns label $1$. We think of each verifier call as the algorithm
guessing an index in $[m]$: the response is $1$ if the guess equals $I$ and
$0$ otherwise.

Suppose the algorithm has made $q$ guesses, all
returning $0$ (i.e.\ $T > q$). It has ruled out at most $q$ indices, so the
posterior on $I$ is uniform over the remaining set, which has size at least
$m - q$. Whatever index the algorithm guesses next is a deterministic
function of its history, so it hits $I$ with probability at most
$\frac{1}{m-q}$. Hence
\[
  \Pr(T > q+1 \mid T > q) \;\ge\; 1 - \frac{1}{m-q}.
\]

Multiplying the survival probabilities,
\[
  \Pr(T > q)
  \;\ge\; \prod_{k=0}^{q-1}\!\left(1 - \frac{1}{m-k}\right)
  \;=\; \prod_{k=0}^{q-1}\frac{m-k-1}{m-k}
  \;=\; \frac{m-q}{m},
\]
where the product telescopes (each numerator cancels the next denominator). Using $\mathbb{E}[T] = \sum_{q \ge 0}
\Pr(T > q)$,
\[
  \mathbb{E}[T]
  \;\ge\; \sum_{q=0}^{m-1}\frac{m-q}{m}
  \;=\; \frac{m+1}{2}
  \;\ge\; \frac{m}{2}.
\]
The algorithm pays at least $\vercost$ per verifier call, so
$\mathbb{E}[J(\Alg;h_I,\Dist)] \ge \vercost\, \mathbb{E}[T] \ge
\tfrac{\vercost}{2} m$, where the expectation is over $I$, the samples, and
$\Alg$'s randomness. Averaging over $I$ produces some $i^\star \in [m]$ with
\begin{equation}\label{eq:alg}
  \mathbb{E}[J(\Alg;h_{i^\star},\Dist)] \;\ge\; \tfrac{\vercost}{2}\, m.
\end{equation}

Set $h^\star = h_{i^\star}$ and $\gamma :=
\vercost/\rewcost$. Combining \eqref{eq:oracle} and \eqref{eq:alg},
\[
  \frac{\mathbb{E}[J(\Alg;h^\star,\Dist)]}{J^\star(h^\star,\Dist)}
  \;\ge\;
  \frac{(\vercost/2)\, m}{\rewcost\, m + \vercost}
  \;=\;
  \frac{1}{2}\cdot\frac{\gamma m}{m + \gamma}
  \;\ge\;
  \frac{1}{4}\min\{m,\gamma\}.
\]
The bound holds for every $m \le \starnum_0$ yields the claimed
$\tfrac{1}{4}\min\{\starnum_0, \gamma\}$.
\end{proof}
\subsection{Extension to Probabilistic Concept Classes} \label{sec:start-num-probalistic}
Next we give an extension of the centered star number \cref{def:starnum} for probabilistic concept classes. 
\begin{definition}
For a probabilistic concept class \(\conceptclass\subseteq[0,1]^{\dataspace}\)
and \(\eta\in(0,1]\), define $\eta$-centered star number of $\conceptclass$ as
\[
    \starnum_{0,\eta}
    :=
    \sup\left\{
    m:
    \exists x_1,\dots,x_m,\ h_1,\dots,h_m\in\conceptclass
    \text{ such that }
    h_i(x_i)\ge\eta,\ 
    h_i(x_j)=0\ \forall j\neq i
    \right\}.
\]
\end{definition}
The following corollary extends \cref{thm:starnum} to the probabilistic concept classes.
\begin{corollary}
For every integer \(m\le \starnum_{0,\eta}\), every sound active search
algorithm has a worst-case instance \((h^\star,\Dist)\) such that
$ \displaystyle
    \frac{\mathbb E[J(\Alg;h^\star,\Dist)]}{J^\star(h^\star,\Dist)}
    \ge
    \frac{\eta}{4}
    \min\left\{
        \starnum_{0,\eta},\frac{\vercost}{\rewcost}
    \right\}.
$
\end{corollary}

\subsection{Connection Between Active Learning and Active Search}
\label{sec:activelearn-search}

This section formalizes the relation between pool-based active learning and
pool-based active search. Active learning aims to output a classifier with small
error, while active search only aims to find one positive example. We show that,
on any instance with sufficiently large positive mass, active learning implies
active search. 

For a distribution \(D \in \ProbMeasures{\mathcal X}\) and two measurable
functions \(f,g:\mathcal X\to\{0,1\}\), define
\[
\operatorname{err}_D(f,g)
:=
\Pr_{X\sim D}\!\left(f(X)\neq g(X)\right).
\]
We use this both for the error of a learner's output relative to a target and
for the disagreement between two concepts. In particular, if \(h_0\equiv 0\),
then
\[
\operatorname{err}_D(h,h_0)
=
\Pr_{X\sim D}(h(X)=1).
\]

\begin{definition}
An \((n,m)\) pool-based active learning algorithm \(\Alg\) proceeds as follows.
First, it receives an unlabeled sample
\[
S_n=(X_1,\ldots,X_n)\sim D^n.
\]
Then, for rounds \(t=1,\ldots,m\), the algorithm adaptively chooses an index
\(I_t\in[n]\) as a function of \(S_n\) and the previously observed labeled
examples
\[
(X_{I_1},Y_{I_1}),\ldots,(X_{I_{t-1}},Y_{I_{t-1}}).
\]
After at most \(m\) label queries, the algorithm outputs a classifier
\(\widehat h:\mathcal X\to\{0,1\}\).
\end{definition}

\begin{definition}
Let \(\mathcal H\subseteq\{0,1\}^{\mathcal X}\) be a binary concept class, and
fix \(\varepsilon,\delta\in[0,1]\). We say that \(\Alg\)
\((\varepsilon,\delta)\)-learns \(\mathcal H\) with unlabeled sample complexity
\(n\) and label query complexity \(m\) if, for every \(h^\star\in\mathcal H\)
and every \(D\in\ProbMeasures{\mathcal X}\),
\[
\Pr\left(
\operatorname{err}_D(\widehat h,h^\star)\le \varepsilon
\right)\ge 1-\delta,
\]
where the probability is over \(S_n\sim D^n\) and any internal randomness of
\(\Alg\).
\end{definition}

\begin{definition}[Pool-based active search]
Let \(\mathcal H\subseteq\{0,1\}^{\mathcal X}\) be a binary concept class, and
fix \(\delta\in[0,1]\). We say that an active search algorithm \(\Alg\) finds a
positive example for \(\mathcal H\) with unlabeled sample complexity \(n\) and
label query complexity \(m\) with probability at least \(1-\delta\) if, for
every \(h^\star\in\mathcal H\) and every \(D\in\ProbMeasures{\mathcal X}\) with
\(\Pr_{X\sim D}(h^\star(X)=1)>0\),
\[
\Pr\!\left(h^\star(X_{\widehat I})=1\right)\ge 1-\delta,
\]
where \(\widehat I\in[n]\) is the index output by \(\Alg\), and the probability
is over \(S_n\sim D^n\) and any internal randomness of \(\Alg\).
\end{definition}

\begin{theorem}[Active learning implies active search]
\label{thm:learning-implies-search}
Let \(\mathcal H\subseteq\{0,1\}^{\mathcal X}\) be a binary concept class
containing the all-zero concept \(h_0\equiv 0\). Suppose that
\(\Alg_{\rm learn}\) is an \((n,m)\) pool-based active learning algorithm that
\((\varepsilon,\delta)\)-learns \(\mathcal H\). Then there exists an
\((n,m)\) pool-based active search algorithm \(\Alg_{\rm search}\) such that,
for every \(h^\star\in\mathcal H\) and every
\(D\in\ProbMeasures{\mathcal X}\) with
\[
p_\star := \Pr_{X\sim D}(h^\star(X)=1) > 2\varepsilon,
\]
we have
\[
\Pr_{h^\star}\!\left(h^\star(X_{\widehat I})=1\right)\ge 1-2\delta .
\]
\end{theorem}

\begin{proof}
The algorithm \(\Alg_{\rm search}\) simulates \(\Alg_{\rm learn}\) on the pool
\(S_n\sim D^n\), forwards each label query to the \(h^\star\)-oracle, and
outputs the first queried index whose label is \(1\). If no such index is
queried, it outputs an arbitrary index. Let \(E\) be the event that no queried
label is \(1\), and let \(\mathbb P_h\) denote the law of the execution under
target \(h\).

First, the output \(\widehat h\) of \(\Alg_{\rm learn}\) distinguishes
\(h_0\) from \(h^\star\). Define
\[
T := \indic{\operatorname{err}_D(\widehat h,h_0)\le\varepsilon}.
\]
Under \(h_0\), the learning guarantee gives
\(\mathbb P_{h_0}(T=1)\ge 1-\delta\). Under \(h^\star\), since
\[
p_\star=\operatorname{err}_D(h^\star,h_0)>2\varepsilon,
\]
the triangle inequality implies that \(T=1\) forces
\[
\operatorname{err}_D(\widehat h,h^\star)
\ge
p_\star-\operatorname{err}_D(\widehat h,h_0)
>
\varepsilon .
\]
Thus \(\mathbb P_{h^\star}(T=1)\le\delta\). By the variational
characterization of total variation,
\[
\|\mathbb P_{h_0}-\mathbb P_{h^\star}\|_{\rm TV}
\ge
\mathbb P_{h_0}(T=1)-\mathbb P_{h^\star}(T=1)
\ge
1-2\delta .
\]

We now upper bound the same total variation by the probability that
\(\Alg_{\rm learn}\) queries a positive point under \(h^\star\). Couple the two
executions using the same pool \(S_n\) and the same internal randomness. If the
\(h^\star\)-execution never observes label \(1\), then all observed labels are
identical to those under \(h_0\). Since the next query is a function only of the
pool, the randomness, and the previously observed labels, the two executions are
identical on \(E\). Hence
\[
\|\mathbb P_{h_0}-\mathbb P_{h^\star}\|_{\rm TV}
\le
\mathbb P_{h^\star}(E^c).
\]
Combining the two bounds gives
\[
\mathbb P_{h^\star}(E^c)\ge 1-2\delta .
\]
On \(E^c\), the simulated learner queries at least one point with label \(1\),
and \(\Alg_{\rm search}\) outputs such an index. Therefore
\[
\Pr_{h^\star}\!\left(h^\star(X_{\widehat I})=1\right)\ge 1-2\delta .
\]
\end{proof}

\begin{remark}[Strong separation]
\label{rem:strong-separation}
Theorem~\ref{thm:learning-implies-search} has no converse: there are
instances on which active search needs one label while active learning
requires $\Omega(d)$.

\paragraph{Construction.}
Let $\mathcal{Z} = \{c, e_1, \ldots, e_d\} \subset \Reals^d$ with $c = (1, \ldots, 1)$
and $e_j$ the $j$-th basis vector, and let $\Dist(c) = \tfrac12$,
$\Dist(e_j) = \tfrac{1}{2d}$. For $w \in \{0, 1, 2\}^d$, define
\[
  h_w(z) := \indic{w^\top z \ge \tfrac32},
  \qquad
  \mathcal H := \{h_w : w \in \{0, 1, 2\}^d\}.
\]
Every $h_w$ is monotone in coordinates with nonnegative weights, and
$\mathcal H$ contains $h_0 \equiv 0$ (take $w = 0$). For
$w \in \{1, 2\}^d$, identifying $S(w) := \{j : w_j = 2\}$ gives
$h_w(c) = 1$ and $h_w(e_j) = \indic{j \in S(w)}$.

\paragraph{Search is easy.}
For every $h^\star \in \mathcal H$ with $h^\star \neq h_0$, we have
$\Dist(\{z : h^\star(z) = 1\}) \ge \Dist(c) = \tfrac12$. A pool of
$n = O(\log(1/\delta))$ unlabeled draws thus contains $c$ with probability
$\ge 1 - \delta$, and a single verifier call on $c$ certifies a positive.
Hence $m = 1$.

\paragraph{Learning is hard.}
Restrict to $\{h_w : w \in \{1, 2\}^d\}$, indexed by $C \subseteq [d]$ via
$w_j = 1 + \indic{j \in C}$, and let $\Alg$ be any $(n, m)$ pool-based active
learner with output $\widehat h$. Draw $C \sim \mathrm{Unif}(2^{[d]})$
independently of the pool and $\Alg$'s randomness, and set
$B_j := \indic{j \in C}$, so $B_1, \ldots, B_d \sim \mathrm{Bernoulli}(\tfrac12)^{\otimes d}$.
Under target $h_C$, a query on $c$ deterministically returns $1$ and a
query on $e_j$ returns $B_j$.

Let $\mathcal V$ denote $\Alg$'s view (pool, randomness, observed labels) and
$T \subseteq [d]$ the random set of indices with $e_j$ queried, so $|T| \le m$.
Conditional on $\mathcal V$, the unqueried bits $\{B_j : j \notin T\}$ remain
$\mathrm{Bernoulli}(\tfrac12)$, and $\widehat h$ is $\mathcal V$-measurable;
hence $\Pr(\widehat h(e_j) \neq B_j \mid \mathcal V) = \tfrac12$ for every
$j \notin T$. Dropping the (only-helps-the-learner) contribution from $c$,
\[
  2d \cdot \operatorname{err}_\Dist(\widehat h, h_C) \,\big|\, \mathcal V
  \;\succeq\; \mathrm{Binomial}\!\bigl(d - |T|,\, \tfrac12\bigr).
\]
Fix $\varepsilon < \tfrac14$ and suppose $m \le d(1 - 4\varepsilon)/2$. Then
$d - |T| \ge d(1 + 4\varepsilon)/2$, and Hoeffding gives
\[
  \Pr\!\bigl(\operatorname{err}_\Dist(\widehat h, h_C) \le \varepsilon \,\big|\, \mathcal V\bigr)
  \;\le\; \exp\!\Bigl(-\tfrac{(1-4\varepsilon)^2}{4(1+4\varepsilon)}\, d\Bigr)
  \quad \text{a.s.}
\]
The tower rule preserves the bound, and for $d$ large enough it falls below
$1/3$. By Yao's principle some target in $\mathcal H$ forces failure with
probability $> 1/3$, so $m = \Omega(d)$ for every constant $\varepsilon < 1/4$.
\qed
\end{remark}

\subsection{Proof of \cref{thm:centered-star-upper}}
\begin{proof}[Proof of \cref{thm:centered-star-upper}]
Soundness is immediate since \(\Alg_{\rm CS}\) returns only after observing verifier label \(1\). Let
\(p:=\Pr_{Z\sim\Dist}(h^\star(Z)=1)>0\). By \cref{lem:binary-oracle-value},
\(J^\star(h^\star,\Dist)=\rewcost/p+\vercost\).

First suppose \(\Alg_{\rm CS}\) enters the verify-all branch, i.e.
 \(\starnum_0>\vercost/\rewcost\). The branch succeeds with geometric
success probability \(p\), so
\[
    \EE[J(\Alg_{\rm CS};h^\star,\Dist)]
    =
    \frac{\rewcost+\vercost}{p}.
\]
Therefore
\[
    \frac{\EE[J(\Alg_{\rm CS};h^\star,\Dist)]}{J^\star(h^\star,\Dist)}
    =
    \frac{(\rewcost+\vercost)/p}{\rewcost/p+\vercost}
    =
    \frac{1+\vercost/\rewcost}{1+(\vercost/\rewcost)p}
    \le
    1+\frac{\vercost}{\rewcost}
    \le
    2\frac{\vercost}{\rewcost},
\]
where the last inequality uses \(\vercost\ge\rewcost\). In this branch,
\(\min\{\starnum_0,\vercost/\rewcost\}=\vercost/\rewcost\), and hence
\[
    \EE[J(\Alg_{\rm CS};h^\star,\Dist)]
    \le
    6\min\left\{\starnum_0,\frac{\vercost}{\rewcost}\right\}
    J^\star(h^\star,\Dist).
\]

Now suppose \(\Alg_{\rm CS}\) enters the batch branch, i.e.
\(\starnum_0\le\vercost/\rewcost\). Since the instance is feasible and \(h_0\in\conceptclass\), we
have \(\starnum_0\ge1\). Let
\[
    n:=\left\lceil \frac{\vercost}{\rewcost}\starnum_0\right\rceil .
\]
By \cref{lem:centered-star-hitting-set}, every batch \(S=(Z_1,\dots,Z_n)\) admits a hitting set
\(I\subseteq[n]\) with \(|I|\le\starnum_0\). If the batch contains a positive point, then
\(A_{h^\star}(S)\neq\emptyset\), so \(I\cap A_{h^\star}(S)\neq\emptyset\), and verifying all indices
in \(I\) finds a positive. Thus one batch succeeds with probability
\(\alpha:=1-(1-p)^n\).

The cost of one batch is at most \(B:=\rewcost n+\vercost\starnum_0\). Since
\(n=\lceil(\vercost/\rewcost)\starnum_0\rceil\), \(\starnum_0\le\vercost/\rewcost\), and
\(\starnum_0\ge1\), we have \(n\le2(\vercost/\rewcost)\starnum_0\). Hence
\(B\le3\vercost\starnum_0\). Also,
\begin{equation}
\label{eq:step-proof-cs-alg}
    \frac{B}{n}
    =
    \rewcost+\frac{\vercost\starnum_0}{n}
    \le
    2\rewcost .
\end{equation}
Let \(x:=np\). Since \(1-p\le e^{-p}\), we have
\(\alpha=1-(1-p)^n\ge1-e^{-np}=1-e^{-x}\). If \(x\le1\), concavity of
\(1-e^{-x}\) gives \(1-e^{-x} \ge x/2\). If \(x>1\), then
\(1-e^{-x}\ge1-e^{-1}\ge1/2\). Hence
\[
    \alpha\ge \frac12\min\{1,np\}.
\]
Therefore
\[
    \EE[J(\Alg_{\rm CS};h^\star,\Dist)]
    \le
    \frac{B}{\alpha}
    \le
    \frac{2B}{\min\{1,np\}}.
\]
If \(np\le1\), then from \cref{eq:step-proof-cs-alg}, we have $\EE[J(\Alg_{\rm CS};h^\star,\Dist)]
    \le
    \frac{2B}{np}
    =
    \frac{2(B/n)}{p}
    \le
    \frac{4\rewcost}{p}$. Then, notice that $ \frac{4\rewcost}{p} \leq
    4J^\star(h^\star,\Dist)
    \le
    6\starnum_0J^\star(h^\star,\Dist),$
where the last inequality uses \(\starnum_0\ge1\). If \(np>1\), then
\[
    \EE[J(\Alg_{\rm CS};h^\star,\Dist)]
    \le
    2B
    \le
    6\vercost\starnum_0
    \le
    6\starnum_0J^\star(h^\star,\Dist),
\]
where the last step uses \(J^\star(h^\star,\Dist)\ge\vercost\). Since in the batch branch
\(\min\{\starnum_0,\vercost/\rewcost\}=\starnum_0\), this gives
\[
    \EE[J(\Alg_{\rm CS};h^\star,\Dist)]
    \le
    6\min\left\{\starnum_0,\frac{\vercost}{\rewcost}\right\}
    J^\star(h^\star,\Dist).
\]
Combining the two branches proves the theorem.
\end{proof}

\begin{lemma} \label{lem:centered-star-hitting-set}
Let $\mathcal{Z}$ be an arbitrary input space, $\conceptclass \in \{0,1\}^{\mathcal{Z}}$, and $n \in \Naturals$ where $n \geq \starnum_0(\conceptclass)$. Let $S \in \mathcal{Z}^n$ be a sequence of length $n$. For every $h \in \conceptclass$, define $A_h(S)=\{i \in [n]~:~h(Z_i)=1\}$. Then there exists \(I\subseteq[n]\) such that \(I\cap A_h(S)\neq\emptyset\) for every \(h\in\conceptclass\) with \(A_h(S)\neq\emptyset\), and \(|I|\le\starnum_0\).
\end{lemma}
\begin{proof}
Let
\[
    \mathcal A_S
    :=
    \{A_h(S):h\in\conceptclass,\ A_h(S)\neq\emptyset\}.
\]
If \(\mathcal A_S=\emptyset\), then \(I=\emptyset\) satisfies the claim. Hence assume
\(\mathcal A_S\neq\emptyset\). Since \([n]\) hits every set in \(\mathcal A_S\), there exists a
minimum-cardinality hitting set \(I\subseteq[n]\), i.e.,
\(I\cap A\neq\emptyset\) for every \(A\in\mathcal A_S\), and \(|I|\) is minimal among all such
sets.

We claim that \(|I|\le\starnum_0(\conceptclass)\). For every \(i\in I\), minimality implies that
there exists \(h_i\in\conceptclass\) such that
\[
    A_{h_i}(S)\cap I=\{i\}.
\]
Indeed, if no such \(h_i\) existed, then every set in \(\mathcal A_S\) hit by \(i\) would also be hit
by \(I\setminus\{i\}\), and every set not hit by \(i\) is already hit by \(I\setminus\{i\}\). Thus
\(I\setminus\{i\}\) would still be a hitting set, contradicting minimality of \(I\).

Write \(I=\{i_1,\dots,i_m\}\). For each \(a\in[m]\), the witness
\(h_{i_a}\) satisfies \(A_{h_{i_a}}(S)\cap I=\{i_a\}\). Hence
\(h_{i_a}(Z_{i_a})=1\), while \(h_{i_a}(Z_{i_b})=0\) for every
\(b\neq a\). Therefore
\[
    h_{i_a}(Z_{i_b})=\indic{a=b}
    \qquad \forall a,b\in[m].
\]
Thus the selected points \(Z_{i_1},\dots,Z_{i_m}\) and the witness concepts
\(h_{i_1},\dots,h_{i_m}\) form a centered star of size \(m\). By definition of
\(\starnum_0(\conceptclass)\), \(m\le\starnum_0(\conceptclass)\). Since
\(m=|I|\), this gives \(|I|\le\starnum_0(\conceptclass)\). Finally, because
\(I\) was chosen as a hitting set, \(I\cap A_h(S)\neq\emptyset\) for every
\(h\in\conceptclass\) with \(A_h(S)\neq\emptyset\).
\end{proof}

\begin{lemma}
\label{lem:binary-oracle-value}
Fix a feasible binary active-search instance \((h^\star,\Dist)\) from \cref{def:active-search-instance}, and let
\(p:=\Pr_{Z\sim\Dist}(h^\star(Z)=1)>0\). Then
\[
    J^\star(h^\star,\Dist)=\frac{\rewcost}{p}+\vercost .
\]
\end{lemma}

\begin{proof}
The distribution-aware oracle generates i.i.d.\ samples from \(\Dist\) until it sees a point \(Z\) with \(h^\star(Z)=1\), verifies this point once, and stops. This gives expected cost \(\rewcost/p+\vercost\). Conversely, any sound policy must generate at least one positive point and must make at least one verifier call. If \(T_+:=\inf\{i\ge1:h^\star(Z_i)=1\}\), where \(Z_i \sim \Dist\), then \(\EE[T_+]=1/p\), and hence every sound policy has expected cost at least \(\rewcost/p+\vercost\).
\end{proof}

\end{document}